\documentclass{article}

\usepackage{PRIMEarxiv}

\usepackage[round]{natbib}

\usepackage[utf8]{inputenc} 
\usepackage[T1]{fontenc}    
\usepackage{hyperref}       
\usepackage{url}            
\usepackage{booktabs}       
\usepackage{tcolorbox}
\usepackage{amsfonts}       
\usepackage{nicefrac}       
\usepackage{microtype}      
\usepackage{xcolor}         

\usepackage{amsmath,amsfonts,bm}
\usepackage{algorithm}
\usepackage{algpseudocode}

\def\eqref#1{equation~\ref{#1}}

\def\floor#1{\lfloor #1 \rfloor}
\def\1{\bm{1}}

\DeclareMathAlphabet{\mathsfit}{\encodingdefault}{\sfdefault}{m}{sl}
\SetMathAlphabet{\mathsfit}{bold}{\encodingdefault}{\sfdefault}{bx}{n}

\def\A{{\mathcal{A}}}

\newcommand{\E}{\mathbb{E}}

\usepackage{cleveref}
\usepackage{graphicx}
\usepackage{multirow}

\usepackage{amsmath}
\usepackage{amssymb}
\usepackage{mathtools}
\usepackage{amsthm}
\usepackage{siunitx}
\usepackage{algorithm}
\usepackage[noEnd=false,indLines=true,beginComment=//~]{algpseudocodex} 

\usepackage{caption}
\usepackage{subcaption}
\usepackage{float}

\theoremstyle{plain}
\newtheorem{theorem}{Theorem}[section]
\newtheorem{proposition}[theorem]{Proposition}
\newtheorem{lemma}[theorem]{Lemma}
\newtheorem{corollary}[theorem]{Corollary}
\theoremstyle{definition}
\newtheorem{definition}[theorem]{Definition}

\theoremstyle{remark}

\def\E{\mathbb{E}}
\def\X{\mathcal{X}}
\def\Y{\mathcal{Y}}

\def\ddefloop#1{\ifx\ddefloop#1\else\ddef{#1}\expandafter\ddefloop\fi}
\def\ddef#1{\expandafter\def\csname bb#1\endcsname{\ensuremath{\mathbb{#1}}}}
\ddefloop ABCDEFGHIJKLMNOPQRSTUVWXYZ\ddefloop
\def\ddef#1{\expandafter\def\csname c#1\endcsname{\ensuremath{\mathcal{#1}}}}
\ddefloop ABCDEFGHIJKLMNOPQRSTUVWXYZ\ddefloop
\def\ddef#1{\expandafter\def\csname v#1\endcsname{\ensuremath{\boldsymbol{#1}}}}
\ddefloop ABCDEFGHIJKLMNOPQRSTUVWXYZabcdefghijklmnopqrstuvwxyz\ddefloop

\def\E{\mathbb{E}}

\def\X{\mathcal{X}}
\def\Y{\mathcal{Y}}
\def\F{\mathcal{F}}
\def\G{\mathcal{G}}

\def\1{\mathds{1}}

\newif\iffeedback
\feedbacktrue

\iffeedback
\newcommand{\varun}[1]{{\color{red}(Varun: #1)}}
\newcommand{\hari}[1]{{\color{olive}(Hari: #1)}}
\newcommand{\ali}[1]{{\color{green}(Ali: #1)}}
\newcommand{\aliq}[1]{{\color{blue}(Ali Q: #1)}}
\else
\newcommand{\varun}[1]{}
\newcommand{\hari}[1]{}
\newcommand{\ali}[1]{}
\newcommand{\aliq}[1]{}
\fi

\usepackage[inline]{enumitem}
\setlist[description]{leftmargin=0in,labelindent=0in, font=\normalfont\bfseries}

\usepackage{soul}

\title{LOTOS: Layer-wise Orthogonalization for Training Robust Ensembles}

\author{
  Ali Ebrahimpour-Boroojeny \\
  UIUC \\
  \texttt{ae20@illinois.edu} \\
   \And
  Hari Sundaram \\
  UIUC \\
  \texttt{hs1@illinois.edu}
    \And
  Varun Chandrasekaran \\
  UIUC \\
  \texttt{varunc@illinois.edu}
}

\begin{document}

\maketitle
 
\begin{abstract}
Transferability of adversarial examples is a well-known property that endangers all classification models, even those that are only accessible through black-box queries. Prior work has shown that an ensemble of models is more resilient to transferability: the probability that an adversarial example is effective against most models of the ensemble is low. Thus, most ongoing research focuses on improving ensemble diversity. Another line of prior work has shown that Lipschitz continuity of the models can make models more robust since it limits how a model's output changes with small input perturbations. {\em In this paper, we study the effect of Lipschitz continuity on transferability rates.} We show that although a lower Lipschitz constant increases the robustness of a single model, it is not as beneficial in training robust ensembles as it increases the transferability rate of adversarial examples across models in the ensemble. Therefore, we introduce \texttt{LOTOS}, a new training paradigm for ensembles, which counteracts this adverse effect. It does so by promoting orthogonality among the top-$k$ sub-spaces of the transformations of the corresponding affine layers of any pair of models in the ensemble. We theoretically show that $k$ does not need to be large for convolutional layers, which makes the computational overhead negligible. Through various experiments, we show \texttt{LOTOS} increases the robust accuracy of ensembles of ResNet-18 models by $6$ percentage points (p.p) against black-box attacks on CIFAR-10. It is also capable of combining with the robustness of prior state-of-the-art methods for training robust ensembles to enhance their robust accuracy by $10.7$ p.p.

\end{abstract}

\section{Introduction}
\label{sec:intro}

Deep learning models are very brittle to input changes:~\citet{szegedy2013intriguing} were the first to carefully craft ``adversarial'' perturbations that result in incorrect classification outcomes. Subsequent research showed that these adversarial examples ``transfer'' to models with different hyper-parameters, and even different hypothesis classes~\citep{papernot2016transferability, liu2016delving}. This transferability property was used to design black-box attacks against models for which only query access is available~\citep{papernot2017practical}. In these attacks, the adversary trains a local model that is ``similar'' to the victim (black-box) model and uses that to find transferable adversarial examples~\citep{xiao2018generating}. To increase the rate of transfer success, a common strategy is to choose those inputs that are adversarial to an ``ensemble'' of models~\citep{liu2016delving,chen2023rethinking}: fooling an ensemble may result in fooling a potentially unseen model. 

Research has also been carried out on understanding and improving model resilience to such attacks~\citep{dong2019evading}. Adversarial robustness is the innate ability of the model to correctly classify adversarial examples~\citep{madry2017towards}. One way of increasing this robustness is utilizing a diverse\footnote{In terms of the parameters and decision boundaries.} ensemble of models~\citep{yang2020dverge,yang2021trs}. This has also been considered as a mitigation to the transferability problem~\citep{pang2019improving,kariyappa2019improving,yang2020dverge,yang2021trs,sitawarin2023defending}; by increasing the diversity among models of the ensemble, the subspace of the adversarial examples that are effective against most of the models within the ensemble shrinks. While empirical evidence is promising, most of these ensemble robustness methods are either computationally expensive or come at a considerable cost to the accuracy of the models and the overall ensemble. In another vein of research, Lipschitz continuity was also shown to be important for robustness~\citep{szegedy2013intriguing,farnia2018generalizable, boroojeny2024spectrum}. Since the model's Lipschitz continuity controls how the predictions change for small changes in the input, it is intuitive that bounding it will improve robustness. It is important to note that neural networks are a composition of multiple layers and therefore the existing works in this area obtain an upper-bound on the overall Lipschitz constant of the network by bounding each layer independently~\citep{szegedy2013intriguing,sedghi2018singular,senderovich2022towards,delattre2023efficient,boroojeny2024spectrum}.

In our work, {\em we investigate the effect of Lipschitz continuity on the transferability of adversarial examples}. We observe that while decreasing the Lipschitz constant makes each model of the ensemble {\em individually more robust}, it makes them less diverse and consequently {\em increases the transferability rate} among them which in turn {\em hinders the overall ensemble robustness} (\S~\ref{sec:motivation},~\Cref{fig:vanilla_compare}). To resolve this adverse effect, we introduce our novel training paradigm, \textbf{Layer-wise Orthogonalization for Training rObust enSembles (\texttt{LOTOS})}, which orthogonalizes the corresponding affine layers of the models with respect to one another. This increases the diversity of the trained models. \texttt{LOTOS} can be combined with any prior method of training diverse ensembles to further boost their robustness.

Through extensive experiments and ablation studies, we show that \texttt{LOTOS} effectively decreases the transferability rate among the models of an ensemble, which leads to a higher robustness against black-box attacks. As we will show, \texttt{LOTOS} is highly efficient for convolutional layers and is very fast compared to prior methods. We also show that it is an effective method for training robust ensembles of heterogeneous architectures, where other state-of-the-art (SOTA) methods are not applicable. Finally, we investigate how \texttt{LOTOS} is able to improve the results of the prior SOTA methods when they are combined.
    
In short, the main contributions of this work are:

\begin{description}
    \item[Lipschitz continuity is not as effective for ensemble robustness.] To the best of our knowledge, we are the first to study the adverse effect of Lipschitz continuity on the transferability rate of adversarial examples (\S~\ref{sec:motivation}). Prior works study individual model and ensemble robustness separately which does not reveal the issues that the former might cause when used in the form of an ensemble. We show the presence of a trade-off between single model robustness and ensemble robustness as the Lipschitz constant of the individual models changes through empirical analysis. This shows the necessity of special treatment when training ensembles to alleviate this trade-off while benefiting the proven effectiveness of Lipschitz continuity in the robustness of individual models.

    \item[A new ensemble training method for robustness.] We introduce \texttt{LOTOS}, a novel ensemble training method to address the aforementioned trade-off. Prior work diversifies the models in the ensemble using either the final outputs~\citep{kariyappa2019improving,pang2019improving}, distilled features~\citep{yang2020dverge}, latent representations~\citep{huang2023fasten}, or the vectorized form of all the parameters~\citep{yang2021trs}. \texttt{LOTOS} considers the corresponding affine layers of the models and orthogonalizes them with respect to one another using a novel component in the loss function. Additionally, we theoretically and empirically show that our method is highly efficient for convolutional layers (\S~\ref{sec:LOTOS}). We demonstrate that \texttt{LOTOS} improves ensemble robustness against black-box attacks at a nominal increase in robust accuracy (in some cases). Also, it can be combined with prior SOTA in training robust models and ensembles to further boost their robustness (\S~\ref{sec:results}).
  
\end{description}

\section{Related Work\vspace{-3mm}}
\label{sec:related}

To increase the diversity of the models in the ensemble,~\citet{kariyappa2019improving} proposed misalignment of the gradient vectors with respect to the inputs. The intuition behind this idea is how several earlier works generate adversarial examples~\cite{goodfellow2014explaining,kurakin2018adversarial,papernot2016limitations}: these methods use the gradient direction with respect to a given input to find the direction that increases the loss function the most. By moving a small amount toward that direction, examples that are similar to the original example but are misclassified by the model are found. In their paper,~\citet{kariyappa2019improving} hypothesize that when this gradient direction is the same for various models of the ensemble, the common subspace of their adversarial examples will be larger because the loss function behaves similarly around the original data. They use cosine similarity to capture this similarity in the direction; by incorporating that for pairs of models in the loss function, they diversify the models in the ensemble to make it more robust.

\citet{pang2019improving} propose a regularizer to increase the diversity in an ensemble by increasing the entropy in non-maximal predictions.~\citet{yang2020dverge} use an adversarial training objective to increase the diversity of the ensemble by making the non-robust features more diverse.~\citet{yang2021trs} suggest that not only does the misalignment of the gradient vectors of the loss with respect to the inputs matters, but also the Lipschitz constant of the gradients (not parameters) of the loss with respect to the inputs has to decrease and propose a heuristic method to achieve this, which outperforms the prior methods.~\citet{zhang2022building} propose a method based on margin-boosting to diversify the models of an ensemble which did not achieve better results than the prior state-of-the-art methods~\citep{yang2020dverge,yang2021trs}. More recent works try to enhance the robustness of the existing methods by incorporating adversarial training against a variety of publicly available models~\citep{sitawarin2023defending} or enhance their time complexity by using faster (but weaker) attacks for data augmentation and enforcing diversity in the latent space~\citep{huang2023fasten}.

There is a different line of work on the orthogonality of the affine layers of the deep learning models which focuses on making the transformation of a single layer orthogonal (i.e., its rows and columns become orthonormal vectors). This helps to preserve the gradient norm of the layer and has been shown to improve the stability and robustness of the models~\citep{trockman2021orthogonalizing,singla2021skew,xu2022lot,singla2021improved,prach2022almost,hu2023recipe}. This notion of orthogonality is different from what we consider in this work; we wish to make the transformation of ``corresponding layers'' from different models {\em orthogonal with respect to each other}.

\section{Motivation}
\label{sec:motivation}

We start with introducing the notations. Then, we define the notion of transferability rate (\S~\ref{sec:defs}) and continue with introducing our conjecture about the effect of smoothness (via Lipschitzness) on transferability rate (\S~\ref{sec:conjecture}).

\textbf{Notation.} Given the domain of inputs $\mathcal{X}$ and $m$ classes $\mathcal{Y} = \{1,2,\dots,m\}$, we consider a multi-class classifier $\mathcal{F}:\mathcal{X} \rightarrow \mathcal{Y}$ and its corresponding prediction function $f(x)$ which outputs the probabilities corresponding to each class (e.g., the outputs of the softmax layer in a neural network). The loss function for model $\mathcal{F}$ is denoted $\ell_\mathcal{F}: \mathcal{X} \times \mathcal{Y} \rightarrow \mathbb{R_+}$; it uses the predicted scores from $f(x)$ to compute the loss given the true label $y$ (e.g., cross-entropy loss). The population loss for model $\mathcal{F}$, which we may also refer to as risk, is defined as $R_\F(x,y) = \E_x [\ell_\mathcal{F}(x,y)$]. When the models are deep neural networks, we refer to a specific layer using superscripts (e.g., $f^{(i)}$ for the $i$-th layer of deep network). A funtion $f(x)$ is $L$-Lipschitz if $\| f(x) - f(x^\prime)\|_2 \leq L \| x-x^\prime \|_2, \forall x,x^\prime \in \mathcal{X}$. For a matrix $A$, the spectral norm is defined as $\|A\|_2 = \mathrm{sup}_{x\neq 0} \frac{\|Ax\|_2}{\|x\|_2}$. Any affine layer $i$ (e.g., dense layer, convolutional layer) with transformation matrix $A$ is $\|A\|_2$-Lipschitz. When we say layer $i$ (with transformation matrix $A$) has been clipped to a value $C$, this means we ensure $\|A\|_2 \simeq C$.

\subsection{Definitions}
\label{sec:defs}

We begin with the definition of an adversarial attack and then formally define the \textit{transferability rate ($T_{rate}$)} of adversarial examples.

\begin{definition}[Attack Algorithm]
    For a given input/output pair $(x,y) \in \mathcal{X} \times \mathcal{Y}$, a model $\mathcal{F}$, and a positive value $\epsilon$, a targetted attack algorithm $\A_\mathcal{F}^{(t)}(x) = x+\delta_x$ minimizes $\ell_\F(x+\delta_x,y_t)$ such that $\|\delta_x\|_2 \leq \epsilon$. An untargeted attack $\A(x)$ maximizes $\ell_\F(x+\delta_x,y)$. 
\end{definition}

\begin{definition}[Transferability Rate]
\label{def:trans}
For an untargeted adversarial algorithm $\A_\mathcal{F}$ and input space $\mathcal{X}$, we define the transferability rate ($T_{rate}$) of $\A_\mathcal{F}(x)$ from $\mathcal{F}$ to another classifier $\mathcal{G}$, as the following conditional probability:
\begin{equation}
T_{rate}(\A_\mathcal{F}, \mathcal{F}, \mathcal{G}) = \mathbb{P}_{(x,y)\in \mathcal{X} \times \mathcal{Y}} \left[\mathcal{G}(\A_\mathcal{F}(x)) \neq y \mid \; \mathcal{F}(x) = \mathcal{G}(x) = y \land \mathcal{F}(\A_\mathcal{F}(x)) \neq y \right].
\end{equation}

For the transferability of a targeted attack algorithm $\A_\mathcal{F}^{(t)}$, and target class $y_t$ this definition is:

\begin{equation}
T_{rate}(\A_\mathcal{F}^{(t)}, \mathcal{F}, \mathcal{G}) = \mathbb{P}_{(x,y)\in \mathcal{X} \times \mathcal{Y}} \left[\mathcal{G}(\A_\mathcal{F}^{(t)}(x)) = y_t \mid \; \mathcal{F}(x) = \mathcal{G}(x) = y \land \mathcal{F}(\A_\mathcal{F}^{(t)}(x)) = y_t \right].
\end{equation}
\end{definition}

Note that~\citet{yang2021trs} have similar definitions of transferability except they use joint probabilities instead of conditional probabilities. Using the conditional probability, we can better isolate the ``transferability'' property we are interested in. By using the joint probability, the $T_{rate}$ will depend on the accuracy of the two models and also the performance of the attack algorithm on the source model $\mathcal{F}$. This will not allow us to have an accurate comparison of the $T_{rate}$ between different settings.

Prior work has shown that there is a trade-off between the accuracy of two models and the $T_{rate}$ of adversarial examples between them~\citep{yang2021trs}; as the two models become more accurate, their decision boundaries become more similar and this increases the probability that an adversarial example generated for one of the models transfers to the other model because of the similarity of their margins around the source sample. In the following section, we introduce our conjecture on the trade-off between the Lipschitzness of the models and their $T_{rate}$.

\subsection{Our Conjecture: Lipschitz Continuity Influences Transferability}
\label{sec:conjecture}

Prior works~\citep{szegedy2013intriguing,farnia2018generalizable,boroojeny2024spectrum} highlight the importance of Lipschitzness on the robustness of a model against adversarial examples. To enforce the Lipschitzness of the deep neural networks, these works enforce the Lipschitzness on each individual layer: since the Lipschitz constant of the composition of all the components of the model (layers, activation functions, etc.) is upper bounded by the multiplication of their Lipschitz constants, this leads to a bound on the Lipschitzness of the whole model. In this work, we follow the same procedure to control the upper bound on the Lipschitz constant of a model. For this, we use FastClip~\citep{boroojeny2024spectrum} which is the current SOTA method for controlling the spectral norm of dense layers and convolutional layers. We use the chosen spectral norm for each individual layer to represent these models in the figures and tables. For example, $C=1$ shows that the spectral norm of all the dense layers and convolutional layers have been clipped to $1$, which makes each of them $1$-Lipschitz. For many architectures such as ResNet-18~\citep{he2016deep}, and DLA~\citep{yu2018deep} this effectively controls the Lipschitz constant of the model because the other constituent components (e.g., ReLU activation, max-pooling, softmax) are $1$-Lipschitz~\citep{goodfellow2014explaining}. For details on controlling the effect of batch norm layers, see~\Cref{sec:bn}.

The prior works have shown that by decreasing the value of the Lipschitz constant $L$ for a model, adversarial attacks achieve a lower success rate. This confirms that decreasing the Lipschitz constant for a model makes it more robust to input perturbations. However, the situation is complicated in an ensemble. While a lower Lipschitz constant promotes robustness for each of the models (of an ensemble), we conjecture that it will increase the $T_{rate}$ by making the classification boundaries of the models more similar. We formalize this intuition below:

\begin{figure}[t!]
\centering
    \includegraphics[width=.99\linewidth]{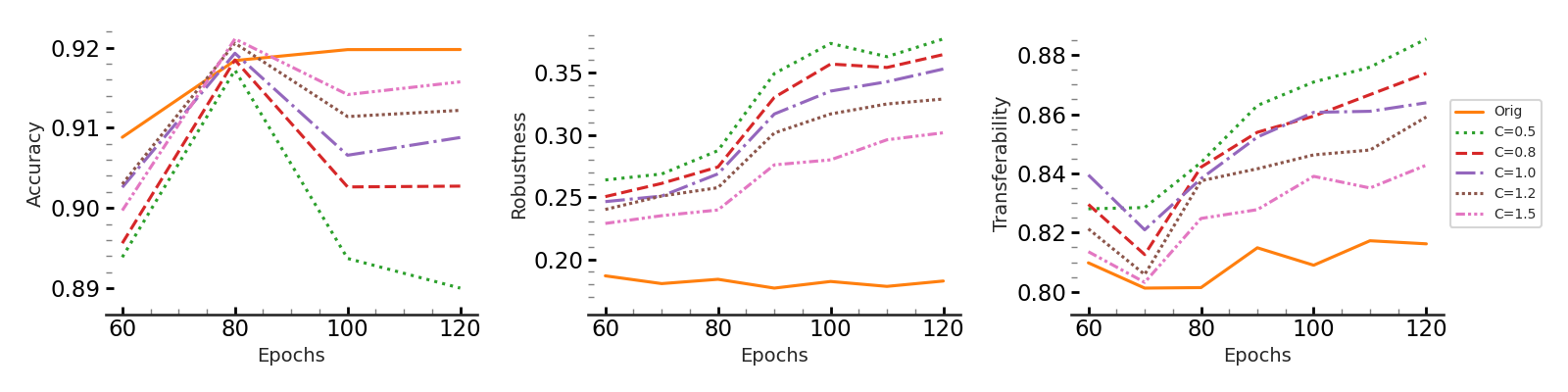} 
\caption{\footnotesize {\bf Accuracy vs. Robust Accuracy vs. Transferability:} Changes in the average accuracy and robust accuracy of \textit{individual} ResNet-18 models, along with the {\em average transferability rate between any pair of the models in each ensemble} as the layer-wise clipping value (spectral norm) changes. As the plots show, although the robustness of \textit{individual} models increases with decreasing the clipping value, the transferability rate among the models increases, which might forfeit the benefits of the clipping in the robustness of the whole ensemble.}
\vspace{-4mm}
\label{fig:vanilla_compare}
\end{figure}

\begin{proposition} 
\label{prop:trans}
    
    Assume $\X = [0,1]^d$ and $\|\delta_x\| \leq \epsilon$. For two models $\F$ and $\G$, if the loss function on both for any $y \in \Y$ is $L$-Lipschitz with respect to the inputs, we have the following inequality:

\begin{align*}
     | R_\F(\A_\mathcal{F}(x),y) - R_\G(\A_\mathcal{F}(x),y) | \leq 2L\epsilon + |R_\F(x,y) - R_\G(x,y)|.
\end{align*}

\end{proposition}

The proof for the proposition can be found in~\Cref{proof:prop}. In Proposition~\ref{prop:trans}, we study the $T_{rate}$ of an adversarial example generated by $A_\mathcal{F}$ to model $\G$ by using the difference in the population loss of the two models on these adversarial examples as a proxy; the lower this difference is, the more likely that the two models will perform similarly on these adversarial examples. To verify this conjecture empirically, in Figure~\ref{fig:vanilla_compare} we show how $T_{rate}$ changes among three ResNet-18~\citep{he2016deep} models without batch norm layers as the layer-wise Lipschitz constant (which governs the upper-bound on the Lipschitz constant of the model as mentioned earlier) changes. Figure~\ref{fig:vanilla_compare_wBN} shows the same behavior for the ResNet-18 models with batch norm layers. For more details see \S~\ref{sec:tradeoff}. 

\begin{tcolorbox}
\noindent{\bf Main Takeaway:} According to Proposition~\ref{prop:trans}, decreasing $L$ (Lipschitz constant), might be an indicator of higher similarity in the risk of the two models on the adversarial examples, which might imply a higher $T_{rate}$. Therefore, although a lower Lipschitz constant could contribute to the robustness of a single model, it might increase the $T_{rate}$ among the models of an ensemble which might hinder the expected benefits of the Lipschitzness. 
\end{tcolorbox}

\section{Layer-wise Orthogonalization for Training Robust Ensembles}
\label{sec:LOTOS}

When clipping the spectral norm of the layers, we are reducing the capacity of the parameters that can be used during the optimization~\citep{neyshabur2017exploring,bartlett2017spectrally}; it is more likely that the parameters of the two models become ``more similar'' when optimized over these constrained spaces. Therefore, to control the Lipschitz constant of each model to make them more ``individually'' robust, {\em and} avoid sacrificing the ``ensemble robustness'', we need to utilize other modifications to enforce the diversity of their decision boundaries on the adversarially perturbed samples.

\begin{tcolorbox}
\noindent{\bf Our Intuition:} The method that we introduce to enforce the diversity 
is based on {\em promoting the orthogonality of the sub-spaces of the corresponding layers of the models that correspond to their top singular vectors}. Since the top singular vectors govern the major part of the transformation by each layer, this orthogonality promotes the difference in the outputs of the corresponding layers from different models. 

\end{tcolorbox}
Affine layers transform the input space such that the sub-space spanned by the top singular vectors will have the most amount of change in the output space for a perturbed input. When the adversary is choosing a perturbation to add to the input, a natural choice would be to choose the direction along the top singular vectors: with the same amount of perturbation, the adversary will get the highest amount of change in the output space for a perturbed input along this direction. Based on this analogy, we consider any two corresponding affine layers $f^{(j)},g^{(j)}$ from a pair of the models $\mathcal{F}$ and $\mathcal{G}$ in the ensemble, whose linear transformations are represented by matrices $A$ and $B$, with the singular value decompositions $A = \sum_{i=1}^d \sigma_i u_i v_i^T$ and $B = \sum_{i=1}^d \sigma_i^\prime u_i^\prime v_i^{\prime T}$, respectively. We define a notion of similarity based on the top-$k$ sub-spaces:
\begin{align}
    \label{equ:sk}
    S_k^{(j)}(f^{(j)},g^{(j)}, \texttt{mal}) := \sum_{i=1}^k w_i ( \mathrm{ReLU} (\|f^{(j)}(v_i^\prime)\|_2 - \texttt{mal}) + \mathrm{ReLU} (\|g^{(j)}(v_i)\|_2 - \texttt{mal})),
\end{align}
\noindent where (a) $w_i$'s are arbitrary weights which are non-increasing with $i$ to emphasize ``more importance'' for the singular vectors corresponding to top singular values, and (b) \texttt{mal} refers to the \textit{maximum allowed length} of the output of each layer when it is given the singular vectors of the other layer as the input (see~\Cref{apx:efficacy} for its effect in practice). Observe that when \texttt{mal} is set to $0$, the value of $S_k(f,g)$ is $0$ if the transformations of $f$ and $g$ are orthogonal in their top-$k$ sub-space (i.e., $\|f^{(j)}(v_i^\prime)\|_2 = \|A v_i^\prime\|_2 = \|\sum_{i=1}^d \sigma_i u_i v_i^T v_i^\prime\|_2 = 0$).

Utilizing this insight, we introduce our technique, \textbf{Layer-wise Orthogonalization for Training Robust Ensembles (\texttt{LOTOS})}. \texttt{LOTOS} promotes the orthogonality among these sub-spaces which leads to different behaviors when perturbing the clean samples along a specific direction. We add this similarity for each pair of corresponding affine layers (dense and convolutional layers) in each pair of models within the ensemble and add them to the cross-entropy loss. More specifically, given an ensemble of $N$ models $\cF_i, \,\, i=\{1,\dots,N\}$ with $M$ layers that would be incorporated in the orthogonalization process, the new loss becomes:
\begin{align}
\mathcal{L}_{\text{train}} = \frac{1}{N}\sum_{i=1}^{N}\mathcal{L}_{\text{CE}}(\mathcal{F}_i(x), y) + \frac{\lambda}{M\,N\,(N-1)}\sum_{z=1}^{N-1}\sum_{j=z+1}^{N} \sum_{l=1}^{M} S_k^{(l)}(f_z^{(l)},f_j^{(l)}, \texttt{mal})
\label{equ:LOTOS-loss}
\end{align}

where $\mathcal{L}_{\text{CE}}(\mathcal{F}_i(x), y)$ is the cross-entropy loss of $\mathcal{F}_i(x)$ given its output on $x$ and the ground-truth label $y$. $\lambda$ controls the effect of the orthogonalization loss and could be adjusted. 

\subsection{Efficiency of LOTOS}
\label{sec:efficiency_LOTOS}

\noindent{\bf Time Complexity:} Note that the number of summands in the orthogonalization loss is $O(N^2 M)$. The computation of $S_{k}^{(t)}$ uses the computed singular vectors of each layer by FastClip~\citep{boroojeny2024spectrum}, which is fast and accurate in practice, and feeds them to the corresponding layer of the other models (see~\Cref{equ:sk}): therefore, it is as if each model has an extra batch of size $N-1$ to process at each iteration, which is relatively small when $N$ is small. In Appendix~\ref{sec:time}, we compare the empirical running time per epoch for ensembles of three models. Although this increase in running time is small, our experiments (\S~\ref{sec:results}) show that performing the orthogonalization for only the first layer would be effective for training robust ensembles (i.e., $M=1$). Therefore, the increase in the training time becomes negligible compared to when the clipping model is used without \texttt{LOTOS}.

\noindent{\bf Highly Efficient for Convolutional Layers:} For the orthogonalization to be effective in~\Cref{equ:LOTOS-loss}, it is necessary to increase the value of $k$ (dimension of orthogonal sub-spaces) because the layers of the DNNs are transformations between high dimensional representations. Therefore, only orthogonalizing the sub-space corresponding to the few top singular values does not guarantee that there is no strong correlation among the remaining top singular vectors (which might correspond to high singular values in each of the models). However, increasing $k$ decreases the computational efficiency, in both compution of the singular vectors~\citep{boroojeny2024spectrum} and computation of the orthogonalization loss in~\Cref{equ:LOTOS-loss}. 

Fortunately, specific properties of the convolutional layers, which are the most common affine layers in DNNs, allow an effective orthogonalization even with very small values of $k$. In~\Cref{theorem-ortho}, we prove even $k=1$ can be effective in orthogonalization with respect to the remaining singular vectors for convolutional layers.

\begin{theorem}
\label{theorem-ortho}
    Given two convolutional layers, $M_1$ and $M_2$ with a single input and output channel and circular padding for which $\textbf{f}$ is the vectorized form of the filter with a length of $T$, and considering $n$ to be the length of the vectorized input, if $\|A v_1^\prime\|_2 \leq \epsilon $, then: 
\begin{align}
\label{equ:upperbound}
    \|A v_p^\prime\|_2 \leq \sqrt{\epsilon^2 + \pi \|\textbf{f}\|_2^2 \, T^2 \frac{p}{n}},
\end{align}

\noindent where $A$ is the corresponding linear transformation of $M_1$ and $v_p^\prime$ is singular vector of $M_2$ corresponding to its $p$-th largest singular value. 
    
\end{theorem}

The proof can be found in Appendix~\ref{proof:theorem}. As~\Cref{theorem-ortho} shows, by orthogonalization of the linear transformation of the convolutional layer $M_1$ (i.e., $A$) and only the first singular vector of $M_2$ (i.e., $v_1^\prime$), so that $\|A v_1^\prime\|_2 \leq \epsilon$, the size of the output of $M_1$ when applied to the remaining singular vectors of $M_2$ (i.e., $\|A v_p^\prime\|_2$) will be upper bounded as shown in~\Cref{equ:upperbound}. This upper-bound depends on the ratio of the ranking of the corresponding singular value to the input size (i.e., $\frac{p}{n}$), which gets smaller for the top singular vectors that have a higher contribution to the transformations. It also depends on the size of the kernel ($T$) which is usually small in models used in practice (e.g., $3^2$ in 2D convolutional layers). Finally, it also depends on the $\ell_2$ norm of the filter values, which can be controlled simply by using weight decay when optimizing the parameters during training. We verify this efficiency of \texttt{LOTOS} when applied to convolutional layers in our experiments.

\section{Results}
\label{sec:results}

We wish to answer the following questions: (1) Does decreasing the Lipschitz constant of the models of an ensemble increase the $T_{rate}$ between them?; (2) does \texttt{LOTOS} decrease the $T_{rate}$ among the models of an ensemble, and does this decrease in the $T_{rate}$ among the models of the ensemble lead to a lower success rate in black-box attacks from other source models?; (3) what are the effects of varying the ensemble size and the number of orthogonalized singular vectors ($k$) on the performance of \texttt{LOTOS}?; (4) is \texttt{LOTOS} still effective when the models of the ensemble are different?; (5) can we combine \texttt{LOTOS} with the prior work on training robust ensembles to provide additional enhancements to robustness?; and (6) can \texttt{LOTOS} be combined with common methods used for increasing the robustness of the models, such as adversarial training? 

As a quick summary, our results show that: (1) decreasing the Lipschitz constant of the models of an ensemble, although make them \textit{individually} more robust, increases the $T_{rate}$ among them (\S~\ref{sec:tradeoff}); (2) \texttt{LOTOS} is indeed effective at reducing the $T_{rate}$ between the models of an ensemble which leads to more robust accuracy against black-box attacks (\S~\ref{exp:layer-wise_ortho}); (3) when using \texttt{LOTOS}, increasing the ensemble size leads to much higher improvement in the robust accuracy (\S~\ref{sec:size}), and changing the number singular values has negligible impact on the transferability (\S~\ref{sec:var-k}); (4) \texttt{LOTOS} is effective even when the ensemble is heterogeneous (\S~\ref{subsec:hetero}); (5) \texttt{LOTOS} in conjunction with \texttt{TRS}~\citep{yang2021trs}, which is one of the state-of-the-art methods in training robust ensebmels, yields better performance than either in isolation (\S~\ref{subse:prior}); and (6) \texttt{LOTOS} can be used together with adversarial training to boost the robustness of the ensemble (\Cref{sec:adv_train}).

\noindent{\bf Attacks:} We use both black-box attacks and white-box attacks in our experiments. The \textbf{white-box} attack is used to evaluate the $T_{rate}$ of adversarial examples between the models in the ensemble; for each ordered pair of the models in the ensemble, the former is used as the source model to generate the adversarial examples and then the $T_{rate}$ of the generated adversarial examples is evaluated on the latter (target model) using Definition~\ref{def:trans}. The average of this value for all the ordered pairs of the models is considered the $T_{rate}$ of the ensemble. A low $T_{rate}$ between the models of the ensemble does not necessarily imply a more robust ensemble. So to evaluate the robustness of ensembles against adversarial attacks, we also use black-box attacks. In the \textbf{black-box} attacks, an independently trained source (surrogate) model (of the same type as the models in the ensemble) is used to generate the adversarial examples; we then measure the robust accuracy of the ensembles against these adversarial examples. For further details on the setup of experiments, please refer to Appendix~\ref{sec:setup}.

\subsection{Robustness vs. Transferability}
\label{sec:tradeoff}

In this section, we evaluate our conjecture from \S~\ref{sec:conjecture} which was motivated by Proposition~\ref{prop:trans}. For this, we compare the ensemble of three ResNet-18 models trained without any modification (represented as \texttt{Orig}) to ensembles of models in which all the layers of each model is $C$-Lipschitz (by controlling the spectral norm). We vary this Lipschitz constant to see how it affects the robustness and transferability. The chosen Lipschitz constant for each ensemble is used to represent that in results: for example, $C=1.0$ shows that the affine layers of the models in the ensemble are all clipped to $1.0$. The clipping of the layers was achieved using FastClip~\citep{boroojeny2024spectrum}. 

Figure~\ref{fig:vanilla_compare} shows the results when the batch norm layers are removed. The first two subfigures show the changes in the average accuracy and robust accuracy of \textit{individual} ResNet-18 models. The rightmost plot shows the average $T_{rate}$ between any pair of the models in each ensemble as the layer-wise clipping value (spectral norm) changes. As the figure shows, although the robustness of \textit{individual} models increases with decreasing the clipping value, the $T_{rate}$ among the models increases.

\begin{figure}[t]
\centering
    \includegraphics[width=.99\linewidth]{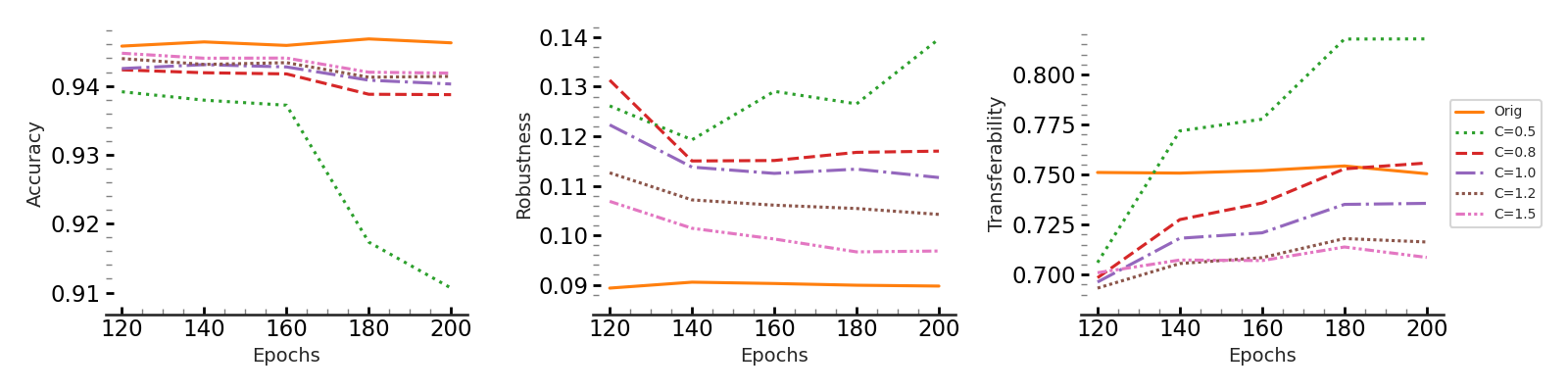}  
\caption{\footnotesize {\bf Accuracy vs. Robust Accuracy vs. Transferability:} Changes in the average accuracy and robust accuracy of \textit{individual} ResNet-18 models (with batch norm layers), along with the average transferability rate between any pair of the models in each ensemble as the layer-wise clipping value changes. As the plots show, although the robustness of \textit{individual} models increases with decreasing the clipping value, the $T_{rate}$ among the models increases, which might forfeit the benefits of the clipping in the robustness of the whole ensemble.}
\label{fig:vanilla_compare_wBN}
\end{figure}

We perform another similar experiment but keep the batch normalization layers in the ResNet-18 models intact (see~\Cref{sec:bn} for more details on evaluating these two cases separately). In~\Cref{fig:vanilla_compare_wBN} we also see that when the batch norm layers are present, \texttt{LOTOS} is still effective; however, the improvement might not be as much as what was observed without batch norm layers. As mentioned earlier, the reason for this lower effectiveness is that the clipping methods for controlling the Lipschitz constant of the models with batch norm layers are less accurate and not as effective. That leads to a less accurate computation of the singular vectors in~\Cref{equ:sk}. Also, batch norm layers are known to adversly affect the robustness of models~\citep{xie2019intriguing,benz2021revisiting} and prior work has pointed out their compensation behavior when controlling the spectral norm of their preceding convolutional layer~\citep{boroojeny2024spectrum}.

\subsection{Efficacy of LOTOS}
\label{exp:layer-wise_ortho}

We first evaluate the effectiveness of \texttt{LOTOS} in decreasing the $T_{rate}$ among the clipped models using the white-box attack. For this, we first use ensembles of three ResNet-18 models and follow the setting explained in~\Cref{sec:setup}.~\Cref{fig:ortho_basic_compare} shows the results for $3$ different methods of training ensembles (\texttt{Orig}, $C=1$, and \texttt{LOTOS}). The left-most subfigure shows the average test accuracy of the \textit{individual} models in each ensemble and the middle subfigure shows the average robust accuracy of the \textit{individual} models in the ensemble. The middle plot shows, as expected, that the individual models in both $C=1$ and \texttt{LOTOS} ensembles are much more robust than the ones in the \texttt{Orig} ensembles because of their Lipschitzness property; however, the \textit{individual} models in \texttt{LOTOS} are not as robust as the ones in $C=1$ because in \texttt{LOTOS} we are enforcing orthogonalization in addition to the Lipschitzness property. Because of this trade-off, as the plot shows, by increasing the value of \texttt{mal} (from $0$ to $0.8$) the robustness of the individual models becomes more similar to the ones in $C=1$ ensembles. As the right-most subfigure shows, the $T_{rate}$ between the models in ensembles trained with \texttt{LOTOS} is much lower than $C=1$ and \texttt{Orig}, and as the \texttt{mal} value decreases (the orthogonalization becomes more strict) the $T_{rate}$ decreases. Figures~\ref{fig:hetero} and~\ref{fig:hetero-wBN} show similar effectiveness in reducing the $T_{rate}$ for ensembles that consist of other models (ResNet-18, ResNet-34, and DLA). See \S~\ref{subsec:hetero} for more details on those figures. Also in~\Cref{apx:efficacy} we show how \texttt{LOTOS} effectively performs the orthogonalization of one layer with respect to the corresponding layer in other models of the ensemble while maintaining the target spectral norm.

\begin{figure}[t!]
\centering
    \includegraphics[width=.99\linewidth]{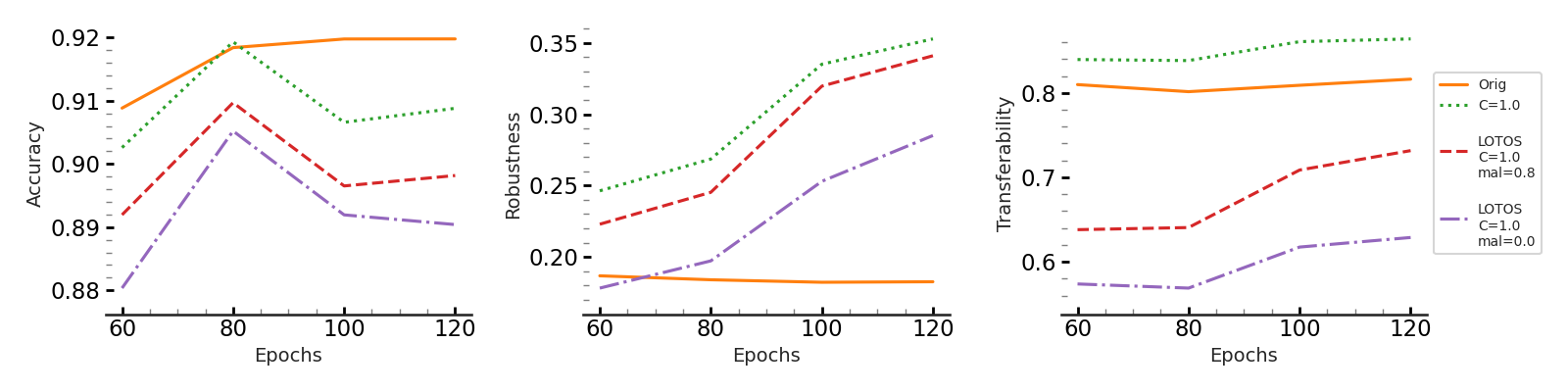}
\caption{\footnotesize  {\bf Reducing transferability while maintaining the benefits of Lipschitzness.} Evaluation of the average test accuracy (left-most plot) and average robust accuracy (middle plot) of the \textit{individual} models in an ensemble of three ResNet-18 models, along with the $T_{rate}$ of adversarial examples between the models of the ensemble using the white-box setting (see Appendix~\ref{sec:setup}). \texttt{LOTOS} keeps the robust accuracy of individual models in the ensemble much higher than those of the \texttt{Orig} ensemble and as \texttt{mal} increases, it becomes more similar to the models in $C=1$. On the other hand, \texttt{LOTOS} leads to a much lower transferability (about $20\%$) and the difference increases as \texttt{mal} decreases (right-most plot). These benefits come at a slight cost to the average accuracy of the individual models (left-most plot).
\vspace{-4mm}}
\label{fig:ortho_basic_compare}
\end{figure}

So far, we observed that \texttt{LOTOS} leads to a noticeable decrease in the transferability at a slight cost in test accuracy and robust accuracy of \textit{individual} models; to make sure that the former overpowers the latter and derives more robust ensembles, we evaluate the robustness of the ensembles against black-box attacks using independently trained surrogate models (see~\Cref{sec:setup} for details). We perform this experiment for both ResNet-18 and DLA models and use both CIFAR-10 and CIFAR-100 for the analysis. The results are presented in~\Cref{tab:bn-dla-resnet}. As the table shows, for each choice of the model architecture and dataset, we train ensembles with either of the 3 methods (i.e., \texttt{Orig}, $C=1$, and \texttt{LOTOS}) and compute their test accuracy and robust accuracy. As the table shows, \texttt{LOTOS} achieves higher robust accuracy in all cases with slight cost to the test accuracy in some cases. 
As~\Cref{tab:nobn-cifar} in~\Cref{sec:var-L} shows, this difference is even more prominent when the batch norm layers are removed because the clipping algorithms are less accurate (see~\Cref{sec:bn}).

\begin{table*}[ht!]
\begin{center}
\begin{small}
\begin{sc}
\resizebox{\textwidth}{!}{
\begin{tabular}{@{} l  c c  c  c  c c  @{}}
 \toprule
 
 & \multicolumn{3}{@{}c}{\textbf{CIFAR-10}} & \multicolumn{3}{@{}c}{\textbf{CIFAR-100}} \\\addlinespace[0.3em]

   &  \texttt{Orig} & $C=1.0$ & \texttt{LOTOS}
 & \texttt{Orig} & $C=1.0$ & \texttt{LOTOS}  \\\addlinespace[0.2em]

     \cmidrule(r){2-4}
    \cmidrule(r){5-7}

 & \multicolumn{6}{@{}c}{Ensembles of ResNet-18 models}   \\ \addlinespace[0.4em]

 Test Acc & $\bf 95.3 \pm 0.06$ & $94.7 \pm 0.24$ & $ 94.6 \pm 0.19$
 & $\bf 77.2 \pm 0.17$ & $76.6 \pm 0.01$ & $76.6 \pm 0.10$ \\\addlinespace[0.3em]

  Robust Acc  & $30.3 \pm 1.63$ & $35.2 \pm 0.72$ & $\bf 36.3 \pm 0.88$
 & $15.2 \pm 0.45$  & $18.9 \pm 0.40$ & $\bf 20.2 \pm 0.47$ \\\addlinespace[0.3em]
 
& \multicolumn{6}{@{}c}{Ensembles of DLA models}   \\ \addlinespace[0.5em]

 Test Acc & $\bf 95.4 \pm 0.12$ & $95.2 \pm 0.05$ & $95.05 \pm 0.09$
 & $77.1 \pm 0.09$ & $\bf 78.8 \pm 0.31$ & $78.3 \pm 0.38$ \\\addlinespace[0.3em]

 Robust Acc & $26.7 \pm 0.58$ & $32.8 \pm 1.28$ & $\bf 34.5 \pm 0.63$  
 & $16.5 \pm 0.78$ & $19.4 \pm 0.32$ & $\bf 21.0 \pm 0.39$ \\\addlinespace[0.3em]

\bottomrule
\end{tabular}
}
\end{sc}
\end{small}
\end{center}
\caption{\footnotesize {\bf Robust accuracy against black-box attacks in ensembles of ResNet-18 models and ensembles of DLA models trained on CIFAR-10 and CIFAR-100 }. The surrogate models are a combination of both original models and clipped models trained with multiple random seeds. The target models are ensembles of three models from each architecture choice that are trained using either of the three training methods.}
\label{tab:bn-dla-resnet}
\end{table*}

\subsection{Ablation Studies}
\label{exp:ablation}

In this section, we explore the effect of $k$ when orthogonalizing the top-$k$ sub-spaces of the convolutional layers. We also investigate the effectiveness of \texttt{LOTOS} as the ensemble size increases. In~\Cref{sec:additional_ablation} we also present additional ablation studies; we evaluate the importance of using the singular vectors in~\Cref{equ:sk} by comparing the transferability when they are replaced with random vectors (\Cref{sec:rand}). We also evaluate the effect of \texttt{mal} value on the transferability (\Cref{sec:var-mal}) and the effect of Lipschitz constant on the strength of the surrogate model in black-box attacks (\Cref{sec:var-L}).

\subsubsection{Increasing Ensemble Size}

To verify the effectiveness of \texttt{LOTOS} as the ensemble size increases, we evaluate the improvement in the robust accuracy of the ensembles against black-box attacks (see Appendix~\ref{sec:setup} for details) when the number of models in the ensemble increase by factors of $3$. Table~\ref{tab:size} shows the results of this experiment for ensembles of ResNet-18 models on CIFAR-10 dataset; the improvement in the robust accuracy when the ensemble size increases is relatively small for the \texttt{Orig} and $C=1$ ensembles. However, when \texttt{LOTOS} is used to train the ensemble, there is a huge improvement in the robustness of the ensemble as the size increases which is due to the increased diversity of the models by orthogonalizing them with respect to one another. The lack of noticeable improvement in the $C=1$ ensembles highlights the fact that making each model individually robust does not necessarily improve the robustness of the ensemble. \texttt{LOTOS} archives this improvement by effectively decreasing the $T_{rate}$ among the models.

\label{sec:size}

\subsubsection{Dimension of Orthogonal Sub-spaces ($k$)}
\label{sec:var-k}

To further evaluate the theoretical observation from~\Cref{theorem-ortho} for general convolutional layers, we use white-box attacks to measure the $T_{rate}$ among the models in an ensemble of three ResNet-18 models when different values of $k$ are used for orthogonalization of the top-$k$ subspace (see~\Cref{equ:sk}). As~\Cref{fig:ablation} (Left) shows, with increasing $k$ there is a slight improvement in the $T_{rate}$ but is still less than one percentage point (compared to $k=1$) even when $k=15$. For $k \geq 20$, we noticed a degradation in the training of the models and the $T_{rate}$, which might be due to over-constraining the models. Given the computational efficiency and the negligible difference in the transferability, we found $k=1$ to be enough and used it in our other experiments.

\begin{figure}[t!]
\begin{minipage}[b]{.48\linewidth}
\centering
\begin{subfigure}{.49\textwidth}
    \centering
    \includegraphics[width=.99\linewidth]{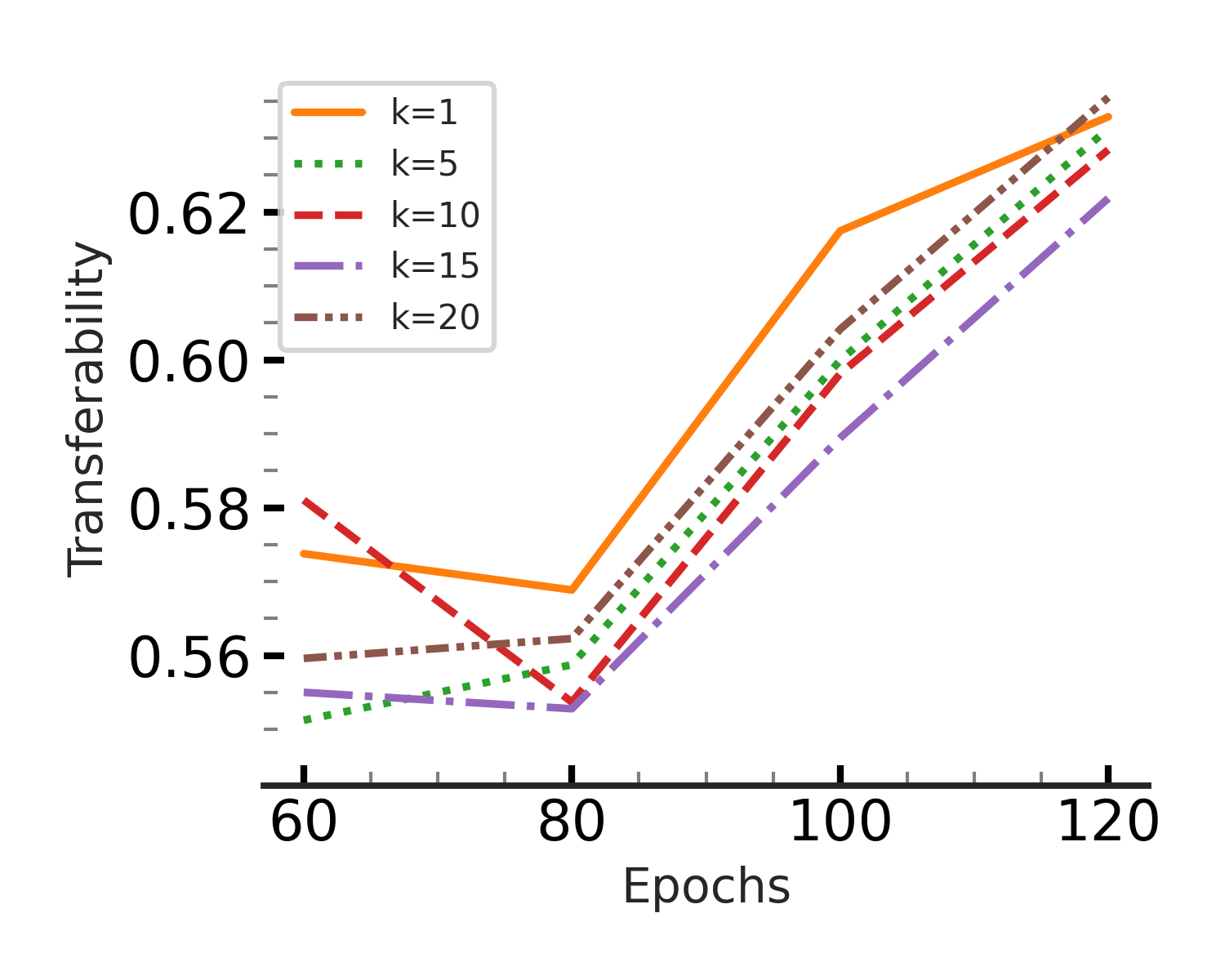}
\end{subfigure}
\begin{subfigure}{.49\textwidth}
    \centering
    \includegraphics[width=.99\linewidth]{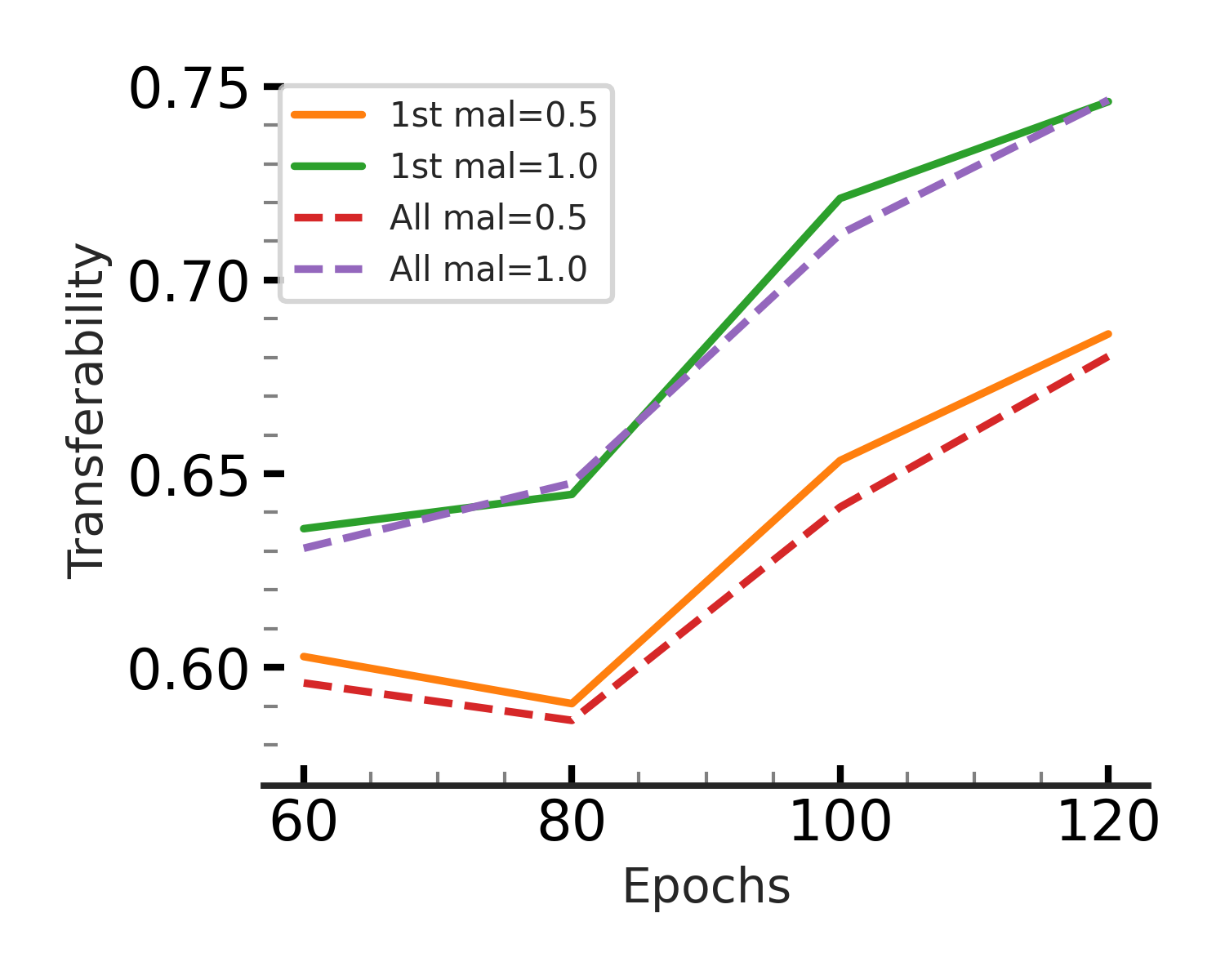}  
\end{subfigure}
\caption{\footnotesize  {\bf (Left) Effect of $k$.} 
As the plot shows, the transferability slightly decreases (up to $1\%$) as $k$ gets larger up to some point (here $k=15$) and then starts to increase ($k=20$).  {\bf (Right) First layer might be enough!} Comparing the effect of applying \texttt{LOTOS} to only the first layer rather than all the convolutional layers. The similarity of the results motivates the effectiveness of \texttt{LOTOS} for heterogeneous models where it can only be applied to the first layers.} 
\vspace{-3mm}
\label{fig:ablation}
\end{minipage}\hfill
\begin{minipage}[b]{.48\linewidth}
\centering
\vskip 0.15in
\begin{center}
\begin{small}
\begin{sc}
\resizebox{\textwidth}{!}{
\begin{tabular}{@{} l  c c  c     @{}}
 \toprule

   &  $N=3$ & $N=9$ & {\footnotesize Improvement}
 \\\addlinespace[0.3em]

    \cmidrule(r){2-4}

 \texttt{Orig} & $16.4 \pm 1.11$ & $17.0 \pm 1.35$ & $3.7\%$

 \\\addlinespace[0.3em]
 
 $C=1$ & $26.4 \pm 0.57$  & $27.5 \pm 0.95$  & $4.2 \%$

 \\\addlinespace[0.3em]

 \texttt{LOTOS} & $\bf 32.2 \pm 0.99$ &  $\bf 45.1 \pm 2.28$ & $\bf 40.1\%$

 \\\addlinespace[0.3em]

\bottomrule
\end{tabular}
}
\end{sc}
\end{small}
\end{center}
\captionof{table}{\footnotesize {\bf Robust accuracy against black-box attacks in ensembles with different sizes}. The target models are ensembles in three different cases; original models (\texttt{Orig}), clipped ones ($C=1$), and trained with \texttt{LOTOS}. As the table shows \texttt{LOTOS} has the highest robust accuracy in each ensemble size and it has the highest rate of improvement when the number of models in the ensemble increases from $3$ to $9$. Note that the robust accuracy of the single \texttt{Orig} model and a single $C=1$ model are $16.1 \pm 0.87$ and $26.2 \pm 0.52$.\vspace{-3mm}}
\label{tab:size}
\vskip -0.1in
\end{minipage}
\end{figure}

\subsection{Heterogeneous Ensembles}
\label{subsec:hetero}

Although the original formulation of \texttt{LOTOS} relies on the similarity of the architecture which allows the layer-wise orthogonalization with respect to other models, when considering different architectures, the first affine transformation of each model is applied to the input data, and therefore has still the same vector space for right singular vectors. Therefore, \texttt{LOTOS} can still be utilized to orthogonalize the first layers on different models. Also, as~\Cref{fig:ablation} (Right) shows, applying \texttt{LOTOS} to only the first layers would still effectively decrease the transferability among the ResNet-18 models in an ensemble of three models trained on CIFAR-10. In this experiment, we consider ensembles of one ResNet-18, one ResNet-34, and one DLA model on CIFAR-10 and report the average accuracy and robust accuracy of individual models, along with the average transferability among them using white-box attacks. As~\Cref{fig:hetero-wBN} shows, \texttt{LOTOS} is effective in reducing the transferability of different models by orthogonalizing only their first convolutional layers, as that would lead to different behaviors on the perturbations of the input.

\begin{figure}[t!]
\centering
    \includegraphics[width=.99\linewidth]{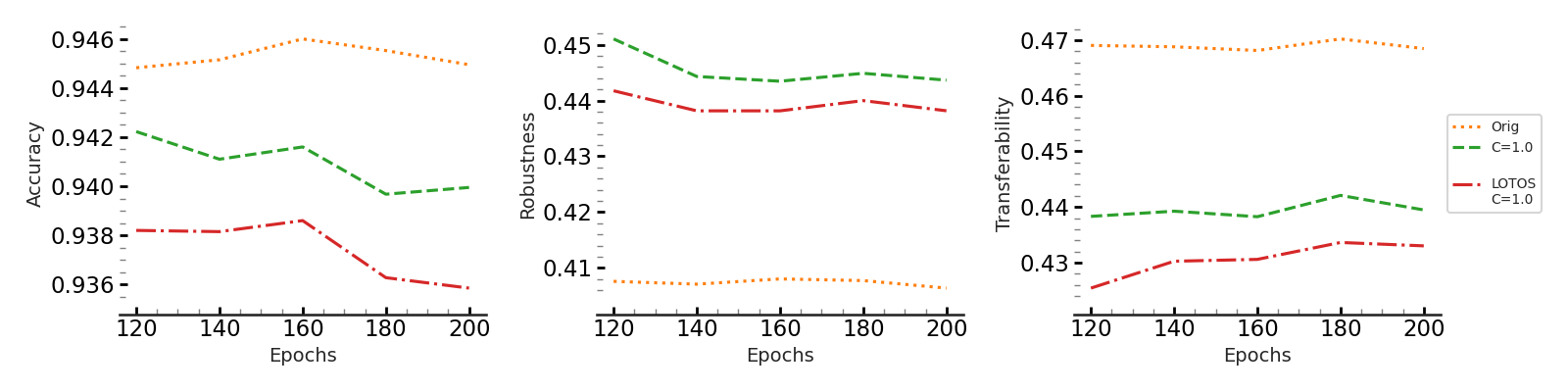}
\caption{\footnotesize {\bf Investigating the effect of \texttt{LOTOS} on the average accuracy and robust accuracy of each of the models of heterogeneous ensembles of DLA, ResNet-18, and ResNet-34 models}, along with presenting the average transferability among any pair of the models in the ensemble as the training proceeds. As the plots show, \texttt{LOTOS} leads to a lower transferability among the models while maintaining the benefits of controlling the Lipschitz constant on the robustness of individual models.}
\vspace{-3mm}
\label{fig:hetero-wBN}
\end{figure}

We also perform a similar experiment but with models without batch norm layers. Since the DLA models cannot be trained without batch norm layers, as observed by prior work~\citep{boroojeny2024spectrum}, we consider ensembles of one ResNet-18 and one ResNet-34 models on CIFAR-10 and report the average accuracy and robust accuracy of individual models, along with the average transferability among them using white-box attacks. As~\Cref{fig:hetero} shows, \texttt{LOTOS} is effective in reducing the transferability of different models by orthogonalizing only their first convolutional layers, as that would lead to different behaviors on the perturbations of the input. As expected, the improvement in the transferability rate is much higher without batch normalization layers (see~\Cref{sec:bn}).

\begin{figure}[ht!]
\centering
 \includegraphics[width=.99\linewidth]{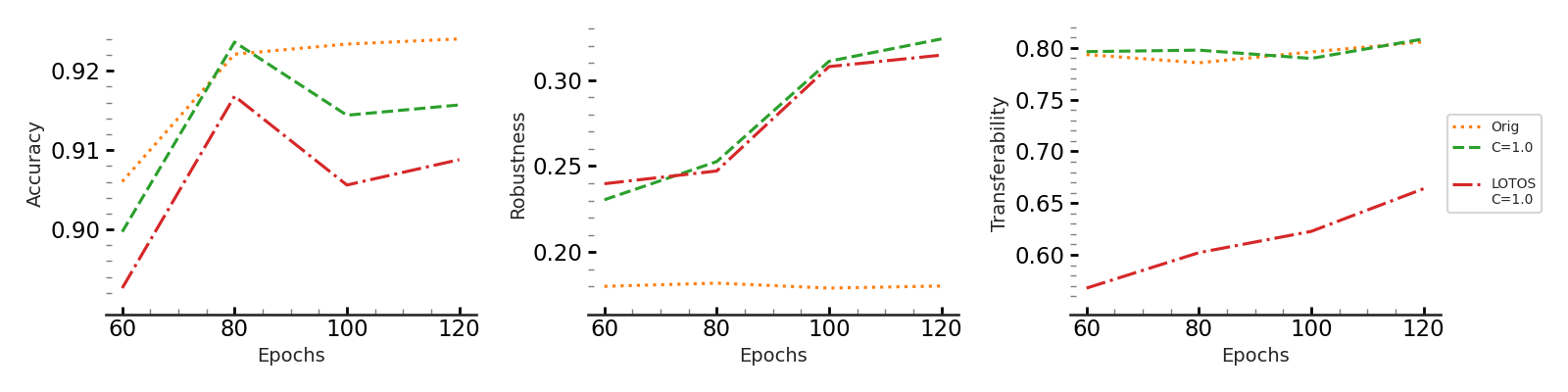}
\caption{\footnotesize {\bf Investigating the effect of \texttt{LOTOS} on the average accuracy and robust accuracy of each of the models of heterogeneous ensembles of ResNet-18 and ResNet-34 models (without batch norm layers)}, along with presenting the average transferability among any pair of the models in the ensemble as the training proceeds. As the plots show, \texttt{LOTOS} leads to a much lower transferability among the models while maintaining the benefits of controlling the Lipschitz constant on the robustness of individual models.}
\vspace{-3mm}
\label{fig:hetero}
\end{figure}

\subsection{Improving Prior Methods}
\label{subse:prior}

For training robust ensembles prior methods have proposed different approaches and in this section, we are interested in observing the effect of \texttt{LOTOS} when combined with these prior works to see if the robustness can be further improved. For this purpose, we use TRS~\citep{yang2021trs}, which is one of the SOTA methods in training robust ensembles and is used as the method of choice by recent works~\citep{sitawarin2023defending}. We report the accuracy and robust accuracy of the ensembles against black-box attacks.~\Cref{tab:bn-cifar} shows the results for the ensembles trained with TRS only (TRS), trained with TRS with clipped models (TRS + $C=1$), and trained with both TRS and \texttt{LOTOS} (TRS + \texttt{LOTOS}). As the results show, for both datasets and both model architectures, the robust accuracy of the combined methods is the highest value. This comes at a slight decrease in the accuracy for CIFAR-10 and a more noticeable degradation for CIFAR-100. Note that for this experiment, we did not perform any hyper-parameter tuning and used both methods with the default parameters.

\begin{table*}[ht!]
\begin{center}
\begin{small}
\begin{sc}
\resizebox{\textwidth}{!}{
\begin{tabular}{@{} l  c c  c  c  c c   @{}}
 \toprule
 
 & \multicolumn{3}{@{}c}{\textbf{CIFAR-10}} & \multicolumn{3}{@{}c}{\textbf{CIFAR-100}} \\\addlinespace[0.3em]

   &  TRS & TRS + $C=1$ & TRS + \texttt{LOTOS}
 & TRS & TRS + $C=1$ & TRS + \texttt{LOTOS}  \\\addlinespace[0.3em]

    \cmidrule(r){2-4}
    \cmidrule(r){5-7}

& \multicolumn{6}{@{}c}{Ensembles of ResNet-18 models}   \\ \addlinespace[0.5em]

 Test Acc & $\bf 94.4 \pm 0.05$ & $94.1 \pm 0.17$ & $92.7 \pm 0.09$
 & $\bf 73.28 \pm 0.46$ & $72.94 \pm 0.29$ & $67.23 \pm 1.22$ \\\addlinespace[0.3em]
 
 Robust Acc & $30.8 \pm 0.65$ & $35.9 \pm 1.35$ & $\bf 41.5 \pm 1.04$  
 & $12.3 \pm 0.53$ & $16.3 \pm 0.57$ & $\bf 20.7 \pm 0.99$ \\\addlinespace[0.3em]

& \multicolumn{6}{@{}c}{Ensembles of DLA models}   \\ \addlinespace[0.5em]

 Test Acc & $\bf 94.72 \pm 0.06$ & $92.79 \pm 0.13$ & $93.18 \pm 0.14$
 & $\bf 72.6 \pm 0.54$ & $63.3 \pm 1.20$ & $66.8 \pm 1.26$ \\\addlinespace[0.3em]
 
 Robust Acc & $31.2 \pm 0.80$ & $32.9 \pm 0.77$ & $\bf 35.3 \pm 0.39$  
 & $23.2 \pm 0.41$  & $23.7 \pm 2.36$ & $\bf 24.3 \pm 1.67$ \\\addlinespace[0.3em]

\bottomrule
\end{tabular}
}
\end{sc}
\end{small}
\end{center}
\caption{\footnotesize {\bf Robust accuracy against black-box attacks in ensembles of ResNet-18 and ensembles of DLA models trained with TRS.} We use ensembles of three models for three different cases; trained with TRS only, trained with TRS while clipping the models, and trained with both TRS and \texttt{LOTOS}. As the results show, using TRS and \texttt{LOTOS} achieves a robust accuracy that is higher than when either of these methods is used.}
\label{tab:bn-cifar}
\end{table*}

\subsection{LOTOS and Adversarial Training}
\label{sec:adv_train}

Adversarial training is the most common method for increasing the robustness of individual models in practice. This method finds a new set of adversarial examples for the training set at each iteration of the training algorithm and combines that with the original training data to feed it to the training method for the next iteration. Iterative training of the model on the adversarial examples, which are perturbed versions of the original training samples, makes them less sussptible to the adversarial examples. In this section, we verify that our proposed training paradaigm not only does not interfere with the robustness of the ensembles of models that are adversarialy trained, but also improves the robustness of the ensemble. For this we train ensembles of ResNet-18 and ensembles of DLA models on CIFAR-10 dataset, and incorporate adversarial training for each model within the ensemble. We repeat this procedure for three settings: 1. no further modification (\texttt{Orig}), 2. clipping each model to $1.0$ ($C=1.0$), and 3. using LOTOS. We evaluate the robust accuracy of these ensembles against blackbox attacks and summarize the results in~\Cref{tab:adv_train}. As the table shows, the ensembles trained with \texttt{LOTOS} achieve a higher robust accuracy while achiving similar test accuracy.

\begin{table*}[ht!]
\begin{center}
\begin{small}
\begin{sc}
\begin{tabular}{@{} l  c c  c   @{}}
\toprule

   &  Adv Train & Adv Train + $C=1$ & Adv Train + \texttt{LOTOS}
 \\\addlinespace[0.3em]

    \cmidrule(r){2-4}

& \multicolumn{3}{@{}c}{Ensembles of ResNet-18 models}   \\ \addlinespace[0.5em]

 Test Acc & $\bf 93.1 \pm 0.00$ & $92.5 \pm 0.16$ & $92.7 \pm 0.09$
 \\\addlinespace[0.3em]
 
 Robust Acc & $60.1 \pm 0.93$ & $60.9 \pm 1.49$ & $\bf 61.7 \pm 1.04$  
 \\\addlinespace[0.3em]

& \multicolumn{3}{@{}c}{Ensembles of DLA models}   \\ \addlinespace[0.5em]

 Test Acc & $90.8 \pm 0.09$ & $93.7 \pm 0.24$ & $\bf 93.9 \pm 0.22$
 \\\addlinespace[0.3em]
 
 Robust Acc & $59.6 \pm 0.93$ & $61.2 \pm 1.32$ & $\bf 62.7 \pm 1.34$  
 \\\addlinespace[0.3em]

\bottomrule
\end{tabular}
\end{sc}
\end{small}
\end{center}
\caption{\footnotesize {\bf Robust accuracy against black-box attacks in ensembles of ResNet-18 and ensembles of DLA models trained with adversarial training} We use ensembles of three models for three different cases; trained with adversarial training only, trained with adversarial training while clipping the models, and trained with both adversarial training and \texttt{LOTOS}. As the results show, using both adversarial training and \texttt{LOTOS} achieves a robust accuracy that is higher than when either of these methods are used.}
\label{tab:adv_train}
\vskip -0.1in
\end{table*}

\section{Conclusions \& Limitations}

We showed there is a trade-off between the robustness of individual models in the ensemble and the transferability rate of adversarial examples among them as the Lipschitz constant of the models changes. This trade-off prevents the expected boost in the robustness of the ensembles of models when they are Lipschitz continuous. Motivated by this observation, we proposed \texttt{LOTOS} that decreases the transferability rate by orthogonalizing the top sub-space of the corresponding layers of different models with respect to one another. We performed a thorough ablation study on the components of our method and showed the effectiveness of \texttt{LOTOS} in boosting the robustness of ensembles. We also showed the negligible computational overhead of our model and the fact that it is only limited by the speed of the underlying method used for controlling the spectral norm of affine layers, which has become practical using recent methods. Another limitation of our method is that it is affected by the degradation of the clipping methods in the models with batch normalization layers.
\clearpage


\bibliography{main}
\bibliographystyle{plainnat}



\appendix
\section*{Appendix}

\section{Proofs}
\label{apx:proofs}

\subsection{Proof for Proposition~\ref{prop:trans}}
\label{proof:prop}

\begin{proof}
    \begin{align*}
    R_\F(\A(x),y) - R_\G(\A(x),y) &= \E_x \ell_\mathcal{F}(\A(x),y) - \E_x \ell_\mathcal{G}(\A(x),y)  = 
    \E_x [\ell_\mathcal{F}(\A(x),y) - \ell_\mathcal{G}(\A(x),y)] \\ 
    &\leq \E_x [\mathrm{sup}_{\|\delta\|_2 < r} (\ell_\mathcal{F}(x+\delta,y) - \ell_\mathcal{G}(x+\delta,y))] \\
    &\leq \E_x [\mathrm{sup}_{\|\delta\|_2 < r} (\ell_\mathcal{F}(x,y) - \ell_\mathcal{G}(x,y)) + 2L\|\delta\|_2]\\
    &\leq \E_x  [(\ell_\mathcal{F}(x,y) - \ell_\mathcal{G}(x,y)) + 2Lr] \\
    &= \int_\X  \ell_\mathcal{F}(x,y) - \ell_\mathcal{G}(x,y) + 2Lr \, d(x,y)\\
    &= 2Lr + R_\F(x,y) - R_\G(x,y).
\end{align*}

We can get a similar inequality starting from $R_\G(\A(x),y) - R_\F(\A(x),y)$ to derive the desired inequality.

\end{proof}

\subsection{Proof for Theorem~\ref{theorem-ortho}}
\label{proof:theorem}

To prove this theorem, we first prove Lemma~\ref{lemma-main} that derives an upper-bound for the difference of the squared values of the largest singular value and any other singular value of a convolutional layer with $1$ input and output channel and circular padding. The main intuition behind the proof of Lemma~\ref{lemma-main} is that the singular values of convolutional layers depend only on the real parts of the first few powers of roots of unity~\citep{boroojeny2024spectrum}. This causes a correlation between the magnitude of the singular values that correspond to the neighboring roots of unity on the real axis.

\begin{lemma}
\label{lemma-main}
    For a convolutional layer with $1$ input and output channel and circular padding, if the length of the vectorized input is $n$, then:
    
\begin{align*}
    \sigma_1^2 - \sigma_p^2 \leq \pi \|\textbf{f}\|_2^2 \, T^2 \frac{p}{n}, 
\end{align*}

\noindent where $\textbf{f}$ is the vectorized form of the convolutional filter with a length of $T$. 

\end{lemma}

\begin{proof}

The rank of the affine transformation of a convolutional layer is $input\,dimention \times \mathrm{min}(c_{in}, c_{out})$~\citep{sedghi2018singular},
 which in the setting of this lemma is equal to $n$. \citet{boroojeny2024spectrum} showed that if the vectorized form of the kernel is given by $ \textbf{f} = [f_0, f_1, \dots, f_{T-1}]$, then we have the following equation for singular values of this convolutional layer:
\begin{align}
    s_j^2 = c_0 + 2\sum_{i=1}^{T-1} c_i \Re(\omega^{j\times i}) ,\,\, j=0,1,2, \dots, n-1,
\end{align}

\noindent where $c_i$'s are defined as:
\begin{align*}
    c_0 &:= f_0^2 + f_1^2 + \dots + f_{T-1}^2,\\
    c_1 &:= f_0f_1 + f_1f_2 + \dots + f_{T-2}f_{T-1},\\
    &\vdots\\
    c_{T-1} &:= f_0f_{T-1}.
\end{align*}

~\noindent and $\omega=\mathrm{exp} (2\pi/n)$ is the basic $n$-th root of unity. The order of $s_j$s (in terms of their magnitude) are unknown apriori and depend on the filter values. Let assume that $\sigma_1 = s_j$ corresponds to the largest singular value. Next, we consider an arbitrary $s_t$ and derive the upperbound for $s_j^2 - s_t^2$:

\begin{align*}
    s_j^2 - s_t^2 = c_0 + 2\sum_{i=1}^{T-1} c_i \Re(\omega^{j\times i}) - c_0 - 2\sum_{i=1}^{T-1} c_i \Re(\omega^{j\times i}) = 2  \sum_{i=1}^{T-1} c_i \left( \Re(\omega^{j\times i}) - \Re(\omega^{j\times i}) \right).
\end{align*}

Now, we need to bound the terms in the summation. We use the fact that $\omega^z = \mathrm{exp}(2z\pi i/n) = \mathrm{cos}\, 2z\pi/n + i \, \mathrm{sin} \, 2z\pi/ n$ and therefore $\Re(\omega^z) = \mathrm{cos}\, 2z\pi/n$:

\begin{align*}
    \Re(\omega^{j\times i}) - \Re(\omega^{j\times i}) &\leq | \Re(\omega^{j\times i}) - \Re(\omega^{j\times i}) | = | \mathrm{cos} \left(\frac{2\pi j\times i}{n}\right) - \mathrm{cos} \left(\frac{2\pi j\times i}{n}\right) | \\
    &= | 2 \mathrm{sin} \left( \frac{(j+t) \pi i}{n} \right) \mathrm{sin} \left( \frac{(t-j) \pi i}{n} \right) | \leq 2 | \mathrm{sin} \left( \frac{(t-j) \pi i}{n} \right) | \\
    &\leq  2 \frac{|t-j| \pi i}{n}, 
\end{align*}

\noindent where the last two last inequalities are due to $sin(x) \leq 1$ and $sin(x) \leq x$, respectively. By using this inequality we can write:

\begin{align*}
    s_j^2 - s_t^2 &= 2  \sum_{i=1}^{T-1} c_i \left( \Re(\omega^{j\times i}) - \Re(\omega^{j\times i}) \right) \leq 2 \sum_{i=1}^{T-1} 2 \, c_i \frac{|t-j| \pi i}{n} \\ 
    &= 4 \frac{|t-j| \pi}{n} \sum_{i=1}^{T-1} c_i \times i \leq 4 \frac{|t-j| \pi}{n} \sum_{i=1}^{T-1} i \times \mathrm{max}_i \, c_i \\
    &= 2 (T-1) T \frac{|t-j| \pi}{n} \times \mathrm{max}_i \, c_i. 
\end{align*}

It is easy to show that $\mathrm{max}_i\, c_i = c_0$; for example, for $c_1$ we can write $2(c_0-c_1) = \sum_{i=0}^{T-2}(f_i - f_{i+1})^2 \geq 0 \implies c_0 \geq c_1$. For $c_i,\, i>1$ a similar justification can be made by only considering the terms that appear in the summands of $c_i$ and using the fact that for the other terms their squared value, which is non-negative, appears in $c_0$. Therefore:

\begin{align*}
    s_j^2 - s_t^2 \leq 2 \, \frac{\pi \, c_0 \, T^2}{n} |t-j|.
\end{align*}

Now considering the set of indices $\mathcal{I} := \{j \pm 1, \dots, j \pm   \floor{p/2} \}$, we know that $s_t,\, t\in \mathcal{I}$ falls within $\,\pi \, c_0 \, T^2 \, p/n$ radius of $s_j$ (which we assumed to represent $\sigma_1$). Therefore, there are at least $|\mathcal{I}|+1=p$ singular values within this radius (including $\sigma_1$). Hence, $\sigma_p$ (the $p$-th largest singular value) should be within this radius as well, and this completes the proof.

\end{proof}

Note that the bound holds for any arbitrary pair of singular values and can be stated in a more general form, as mentioned in Corollary~\ref{cor-anypair}.

\begin{corollary}
\label{cor-anypair}
For the setting of Lemma~\ref{lemma-main}, with sorted singular values $\sigma_1, \dots, \sigma_n$, the following inequality holds:
\begin{align*}
    \sigma_j^2 - \sigma_{j+p}^2 \leq \pi \|\textbf{f}\|_2^2 \, T^2 \frac{p+1}{n}, \,\,\, j,j+p \in [1,\dots,n].
\end{align*}

\end{corollary}

Now, using Corollary~\ref{cor-anypair}, we can easily prove~\Cref{theorem-ortho} and show that for two convolutional layers, orthogonalization of
only the singular vector corresponding to their largest singular values has a similar orthogonalization
effect for the remaining singular vectors and this effect increases for the top singular vectors.

\begin{proof} {\bf (\Cref{theorem-ortho}).}
    Since $M_1$ and $M_2$ are convolutional layers with circular padding and one input channel and output channel, their linear transformation can be represented by circulant matrices of rank $n$ (dimension of the input), $A_1$ and $A_2$~\citep{goodfellow2014explaining}. Singular vector matrix of any circulant matrix is equal to the Fourier matrix of size $n$, and the singular vector corresponding to singular value $s_j$ from unordered list of singular values in Lemma~\ref{lemma-main} can be written as $\frac{1}{\sqrt{n}}[\omega^{j\times 0},\omega^{j\times 1},\dots, \omega^{j\times (n-1)}]^T$~(\cite{gray2006toeplitz} Theorem 3.1). Therefore, the singular vectors of $A_2$ (i.e., $v_i^\prime$s) are simply a different ordering of the singular vectors of $A_1$ (i.e., $v_i$s). Hence, each $\|A_1 v_i^\prime\|_2$ is equal to some singular value $s_j$ of $A_1$. Assuming that $\|A_1 v_1^\prime\|_2 = s_j$ and $\|A_1 v_i^\prime\|_2 = s_{j+i-1}$, by making $\|A_1 v_1^\prime\|_2 \leq \epsilon$ and , using Corollary~\ref{cor-anypair}, we concolude the inequality of interest.
\end{proof}


\section{Experiments}

In this section, we first provide more details on the setting of our experiments (\Cref{sec:setup}) and explain how we control for the effect of batch norm layers in our experiments (\Cref{sec:bn}). We show how \texttt{LOTOS} effectively performs the orthogonalization among the layers of the models in an ensemble in~\Cref{apx:efficacy} and present the results on three other ablations studies on our proposed method in~\Cref{sec:additional_ablation}. Please contact the authors for the code (implementation of \texttt{LOTOS} and the experiments).

\subsection{Setup}
\label{sec:setup}

In this section, we elaborate on the setup of our experiments.

\noindent{\bf Compute Infrastructure:} We used NVIDIA A40 GPUs for our experiments except for the experiments in \S~\ref{subse:prior} that involved training with \texttt{TRS} method where we used NVIDIA A100 GPUs. Using 32GB of RAM was enough to perform our experiments.

\noindent{\bf Datasets and Models:} In all the experiments for evaluating the efficacy of our model, either in isolation or in combination with prior methods, we use both CIFAR-10 and CIFAR-100 datasets~\citep{krizhevsky2009learning}. The models we use in these experiments consist of ResNet-18~\citep{he2016deep}, ResNet-34 (in the experiment on heterogeneous ensembles), and DLA~\citep{yu2018deep}\footnote[1]{According to https://github.com/kuangliu/pytorch-cifar they achieve superior results on CIFAR-10.}. For more exploratory ablation studies (e.g., effect of the parameter $k$ and ensemble size) we limit the experiments to ResNet-18 models on CIFAR-10. 

\noindent{\bf Attack Details:} We use both black-box attacks and white-box attacks in our experiments. 

$\circ$ The white-box attack is used to evaluate the transferability rate of adversarial examples between the models in the ensemble; for each ordered pair of the models in the ensemble, the first one is used as the source model to generate the adversarial examples and then the transferability rate of the generated adversarial examples is evaluated on the second model (target model) using Definition~\ref{def:trans}. The final transferability rate for the ensemble is the average of the transferability rate for all the ordered pairs of the models. To make the results more accurate, we repeat this for $3$ different ensembles trained from scratch with different random seeds and report the average values.

$\circ$ For the black-box attacks, an independently trained source (surrogate) model (of the same type as the models in the ensemble) is used to generate the adversarial examples; we then measure the robust accuracy of the ensembles against these adversarial examples i.e., robust accuracy is the accuracy on the adversarial samples for which the model correctly predict the original versions. We use both the original models and their clipped versions as the surrogate models to generate adversarial examples. To make the results more accurate, we train $3$ source models with different random seeds (for each of the two types of surrogate models) and $3$ ensembles with different random seeds and compute the average of the robust accuracies over the $18$ different cases of choosing a source model and a target ensemble. For both white-box attacks and black-box attacks we use PGD-50~\citep{madry2017towards} attack to generate the adversarial examples.

\subsection{Batch Normalization Layers}
\label{sec:bn}

Controlling the Lipschitz constant of the models is more complicated in the presence of batch norm layers. For analyzing the properties of the Lipschitz models, some prior works ignore the batch norm layers~\citep{miyato2018spectral} and some remove them from the models~\citep{sedghi2018singular}. Some other works modify the parameters of the batch norm directly to ensure Lipschitzness~\citep{gouk2021regularisation,senderovich2022towards,delattre2023efficient}.~\citet{boroojeny2024spectrum} show that batch norm layers show a compensating behavior when the spectral norm of it preceding convolutional layer is controlled so ignoring them does not help with bounding the Lipschit constant of the model. They also show that modifying the parameters directly leads to poor training and test accuracy. They instead propose an approximate method for controlling the Lipschitz constant of the composition of the convolutional and batch norm layers which works better in practice but still are not exact and as accurate. Therefore, for analyzing the true effect of the Lipschitz constant on transferability rate and ablation studies we perform the experiments on models without their batch norm layers, in addition to performing the same experiments on the efficacy of our method in increasing the robustness of ensembles of the original models (including their batch norm layers) when using the approximate methods for controlling the Lipschitz constant in the presence of batch norm layers.

\subsection{Efficacy of LOTOS (cont.)}
\label{apx:efficacy}

As mentioned previously, the goal of the orthogonalization loss of \texttt{LOTOS} in~\Cref{equ:LOTOS-loss} is to keep the size of the output of a layer from each of the models in the ensemble when applied to the largest singular vector of the corresponding layer from the other models below the chosen \texttt{mal} value. This goal has to be achieved while enforcing the spectral norm of each layer to the target clipping value. In this section, we look at some of the layers of the ResNet-18 models in an ensemble that is trained with \texttt{LOTOS} on CIFAR-10, to see how the aforementioned values change during the training of the model. We randomly choose three of the layers from different parts of the models and evaluate them in the course of training for two different chosen values for \texttt{mal}. Figure~\ref{fig:mal_0.5} shows the results when $C=1.0$ and $\texttt{mal} = 0.5$. Each row shows the results for one specific layer from the three models. The leftmost subfigure shows the spectral norm of that layer for each of the models. As the figure shows, the clipping method effectively enforces the spectral norm to be almost $1$ and therefore makes that layer $1$-Lipschitz. The next three subfigures, each show the output of that layer from one of the three models when applied to the largest singular vectors of the corresponding layer from the other two models. As the plots show, these values are effectively controlled to be less than the chosen \texttt{mal} value. 

Figure~\ref{fig:mal.01} shows the results for the same three layers when the ensemble is trained with $\texttt{mal}=0.01$. As the plots show the spectral norm is effectively controlled to be almost $1$, while the output of the layers on the largest singular vectors of the other layers is made much smaller than in the previous case. However, note that for some of these layers, this size might increase up to $0.1$, which is larger than the \texttt{mal} value and the reason is the specific structure of the convolutional layers which does not allow them to accept arbitrary spectrums as noticed by prior work~\citep{boroojeny2024spectrum}.

\begin{figure}
    \centering
    \begin{subfigure}{.99\linewidth}
        \centering
        \includegraphics[width=.99\linewidth]{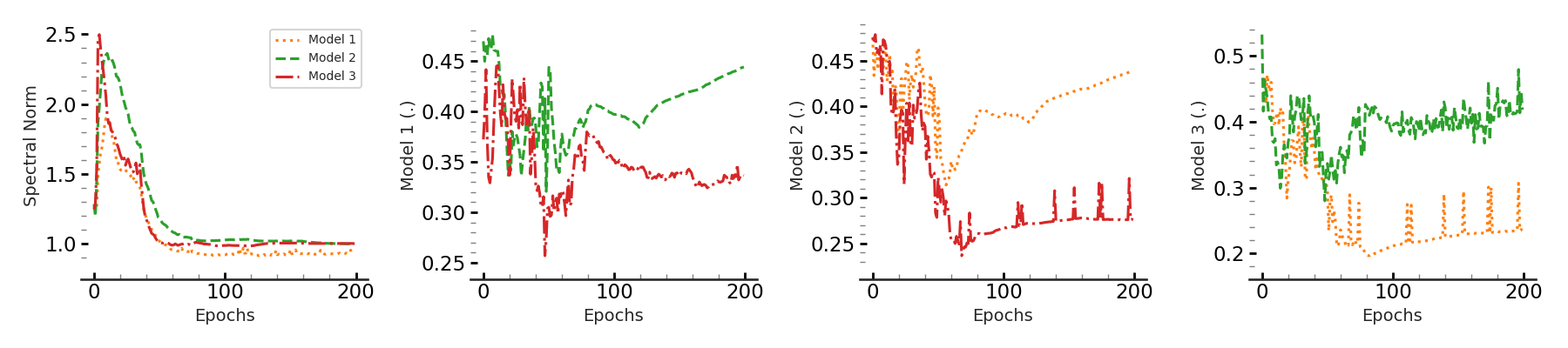}
    \end{subfigure} %

    \hfill

    \begin{subfigure}{.99\linewidth}
        \includegraphics[width=.99\linewidth]{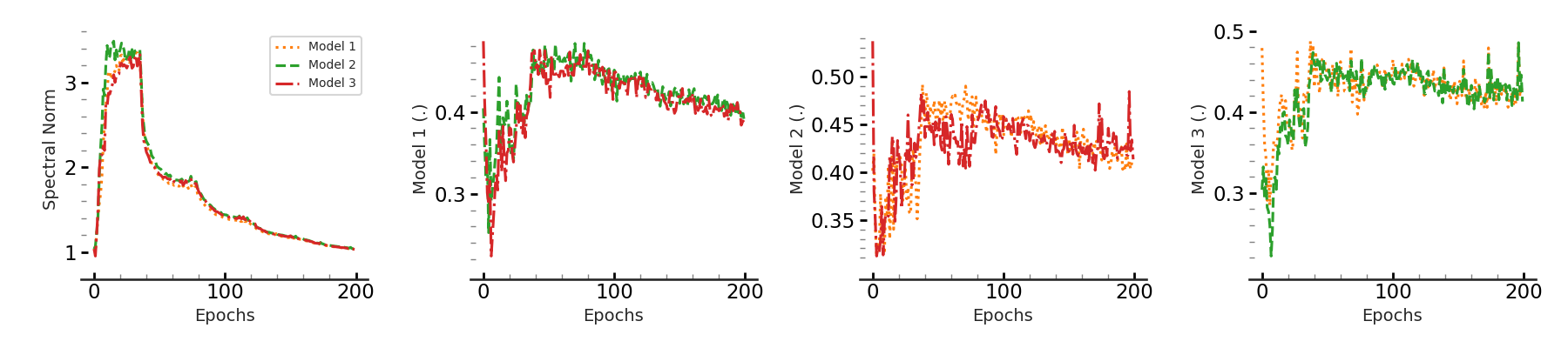}
    \end{subfigure}

    \hfill

    \begin{subfigure}{.99\linewidth}
        \centering
        \includegraphics[width=.99\linewidth]{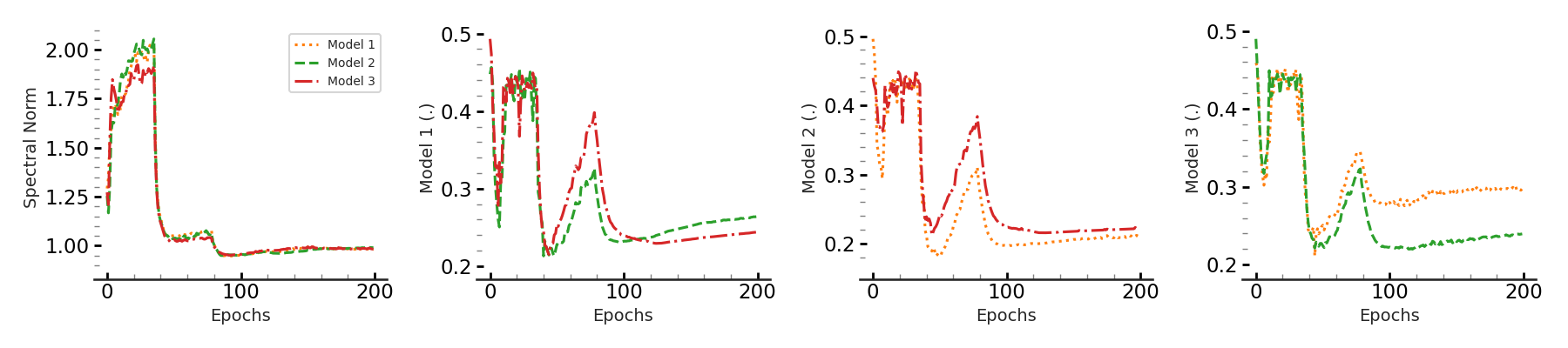}
    \end{subfigure}

    \caption{\footnotesize Each row represents the results for a specific layer from the three ResNet-18 models in an ensemble trained with \texttt{LOTOS} with $C=1$ and $\texttt{mal} = 0.5$. The leftmost subfigure shows the spectral norm of that layer for each model which is enforced to be $1$-Lipschitz. The next three subplots show the size of the outputs of that layer from each model when applied to the largest singular vectors of the corresponding layer from the other models; {\bf \texttt{LOTOS} effectively keeps these values below the chosen value for \texttt{mal}}. 
    \label{fig:mal_0.5}}
\end{figure}

\begin{figure}[t!]
    \centering
    \begin{subfigure}{.99\linewidth}
        \centering
        \includegraphics[width=.99\linewidth]{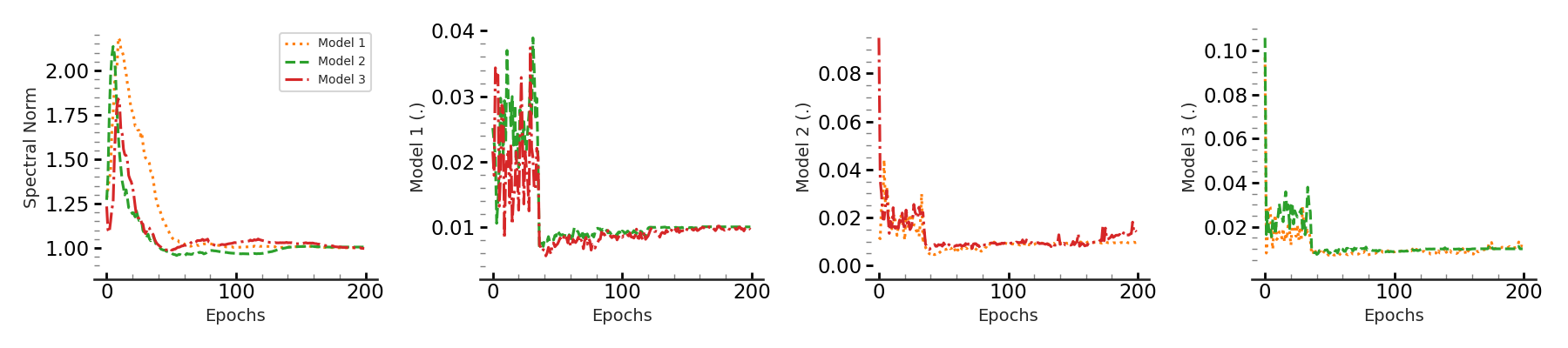}
    \end{subfigure} %

    \hfill

    \begin{subfigure}{.99\linewidth}
        \includegraphics[width=.99\linewidth]{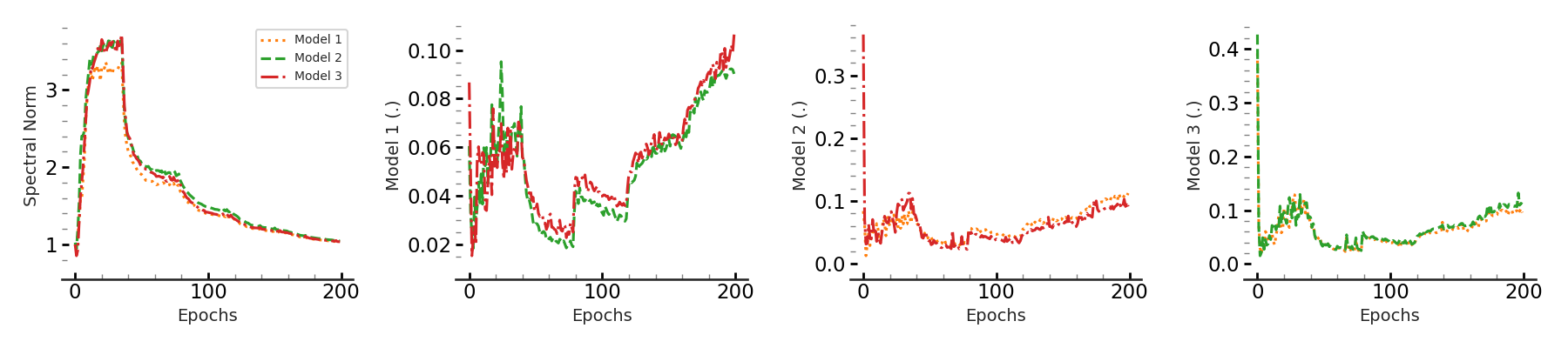}
    \end{subfigure}

    \hfill

    \begin{subfigure}{.99\linewidth}
        \centering
        \includegraphics[width=.99\linewidth]{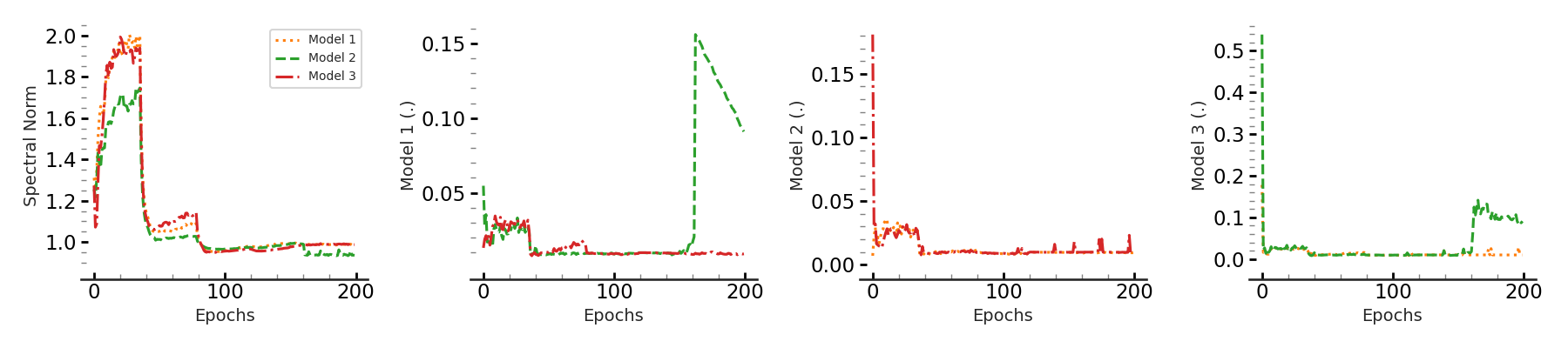}
    \end{subfigure}

    \caption{\footnotesize Each row represents the results for a specific layer from the three ResNet-18 models in an ensemble trained with \texttt{LOTOS} with $C=1$ and $\texttt{mal} = 0.01$. The leftmost subfigure shows the spectral norm of that layer for each model which is enforced to be $1$-Lipschitz. The next three subplots show the size of the outputs of that layer from each model when applied to the largest singular vectors of the corresponding layer from the other models; {\bf \texttt{LOTOS} effectively makes these values much smaller but because the limitations of convolutional layers, they might not become smaller than the \texttt{mal} value in some cases}. 
    \label{fig:mal.01}}
\end{figure}

\subsection{Running Time}
\label{sec:time}

In \S~\ref{sec:efficiency_LOTOS}, we saw that the computation of orthogonalization loss in \texttt{LOTOS} is the same as running each model on $N-1$ additional batch. Table~\ref{tab:time} shows a comparison between the training times of ensembles of three ResNet-18 or DLA models on an NVIDIA A40 GPU. We report the time per epoch for training ensembles without clipping, with clipping, and with \texttt{LOTOS}. For the latter two versions, we also report the increased factor in time compared to the training time of the regular ensembles. These times are shown in three groups: 1. when no additional robustification method is used, 2. when  TRS~\citep{yang2021trs} is also used for training the ensembles, and 3. when Adversarial Training (Adv)~\citep{madry2017towards} is used in training of each individual model in the ensemble. As the table shows the time increase caused by the orthogonalization loss in \texttt{LOTOS} (see~\Cref{equ:LOTOS-loss}) is negligible while being very effective in diversification among the models and decreasing the transferability rate among them.

\begin{table*}[th!]
\begin{center}
\begin{small}
\begin{sc}
\resizebox{\textwidth}{!}{
\begin{tabular}{@{} l  c c  c  c c c c c c @{}}
 \toprule


 \multicolumn{1}{c}{\scriptsize \textbf{}} & 
 \multicolumn{1}{c}{\scriptsize \texttt{Orig}} & 
 \multicolumn{1}{c}{\scriptsize $C=1$} & 
 \multicolumn{1}{c}{\scriptsize \texttt{LOTOS}} & 
 \multicolumn{1}{c}{\scriptsize TRS} & 
 \multicolumn{1}{c}{\scriptsize TRS + $C=1$} & 
 \multicolumn{1}{c}{\scriptsize TRS + \texttt{LOTOS}} &
 \multicolumn{1}{c}{\scriptsize Adv} & 
 \multicolumn{1}{c}{\scriptsize Adv + $C=1$} & 
 \multicolumn{1}{c}{\scriptsize Adv + \texttt{LOTOS}} 
 \\\addlinespace[0.3em]

 \cmidrule(r){2-4}
 \cmidrule(r){5-7}
 \cmidrule(r){8-10}



 ResNet-18 & $33.2  $ & $74.9 $ {\tiny $\times 2.3$} & $79.3$ {\tiny $\times 2.4$}& $158.2  $ & $224.4$ {\tiny $\times 1.4$}  & $227.3$ {\tiny $\times 1.4$}  & $312.6$ & $479.2$ {\tiny $\times 1.5$} & $485.2$ {\tiny $\times 1.5$}
 \\\addlinespace[0.3em]

 DLA  & $63.1$  & $155.4$ {\tiny $\times 2.5$} & $165.4$ {\tiny $\times 2.6$} & $326.2 $ & $466.1$ {\tiny $\times 1.4$} & $477.4$ {\tiny $\times 1.5$} & $758.6.2$ & $942.5$ {\tiny $\times 1.2$} & $949.2$ {\tiny $\times 1.2$}
 \\\addlinespace[0.3em]

 \bottomrule
\end{tabular}
}
\end{sc}
\end{small}
\end{center}
\caption{\footnotesize {\bf }Time (in seconds) per epoch for training ensembles of three ResNet-18 models using either of the methods investigated in this paper. These values are computed on an NVIDIA A40 GPU. As the table shows the orthogonalization loss used in \texttt{LOTOS} makes a negligible change in the computation time of the ensembles with clipped models ($C=1$). \vspace{-3mm}}
\label{tab:time}

\end{table*}

\subsection{Additional Ablation Results}
\label{sec:additional_ablation}
In this section we present the results on three other ablations studies (in addition to the ablation studies presented in \S~\ref{exp:ablation}) on our proposed method; First, we show the importance of using singular vectors of the other layers in~\Cref{equ:sk} by comparing the results when they are replaced by random vectors (\Cref{sec:rand}). We then investigate the effect of changing \texttt{mal} value in~\Cref{equ:LOTOS-loss} in the transferability rate (\Cref{sec:var-mal}). Finally, we investigate the effect of changing the clipping value of the source model in black-box attacks (\Cref{sec:var-L}).

\subsubsection{Largest Singular Vector vs. Random Vectors}
\label{sec:rand}

\texttt{LOTOS} orthogonalizes the corresponding layers of the models by penalizing the size of the output of each layer when it gets the largest singular vectors of the corresponding layers in the other models as its inputs. To show that the decrease in the transferability rate presented in \S~\ref{exp:layer-wise_ortho} is indeed the result of orthogonalization of the subspaces spanned by the top singular vectors, we perform an experiment in which we use randomly chosen vectors in~\Cref{equ:sk} instead. As~\Cref{fig:var-mal} shows, using the transferability rate when random vectors are used is similar to the clipped model without \texttt{LOTOS}. The use of random vectors for the models still causes slight decrease in the transferability rate as it introduces new random differences in the training of the models, but not as much as using the top singular vectors which leads to orthogonalization of the models with respect to one another.

\begin{figure}[t!]
\centering
\begin{subfigure}{.52\textwidth}
    \centering
    \includegraphics[width=.99\linewidth]{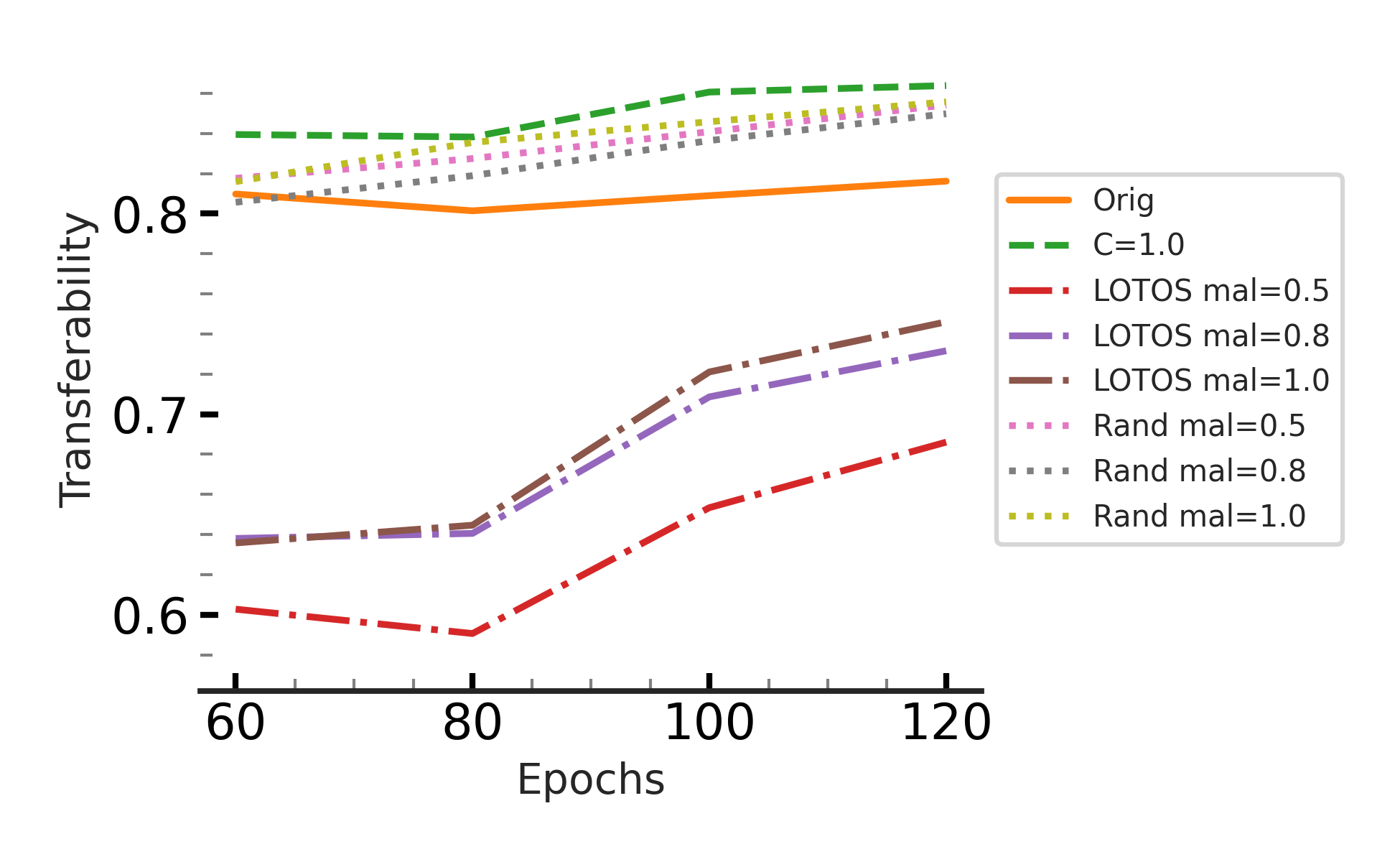}  
\end{subfigure}
\caption{\footnotesize  {\bf Do random vectors work?} We use random vectors in~\Cref{equ:sk} instead of the largest singular vectors to verify the effect of orthogonalization on decreasing the transferability. As the plot shows the ensembles that use random vectors rather than the largest singular vectors of the layers of other models perform similarly to the clipped model without \texttt{LOTOS} ($C=1.0$). It still shows slight improvement because using different random vectors still provides some diversity among the models but not as much as using the top singular vectors. } 
\label{fig:var-mal}
\end{figure}

\subsubsection{Effect of \texttt{mal}}
\label{sec:var-mal}

In~\Cref{equ:LOTOS-loss}, by decreasing the maximum allowed length (\texttt{mal}) for the size of the output of a layer when top singular vectors of the other layers are given as input, we enforce a higher degree of orthogonalization, and therefore expect to see a more decrease in the transferability rate among the models. As~\Cref{fig:ablation} (Right) and~\Cref{fig:var-mal} show, that is indeed the case. However, based on our experiments, decreasing \texttt{mal} to very low values, decreases the robustness and accuracy of individual models. We found the value of $0.8$ to be a good trade-off between the two for increasing the robustness of ensembles and used that for our experiments.

\subsubsection{Changing the Lipschitz Constant of the surrogate Model}
\label{sec:var-L}

 In this experiment, we evaluate the effect of Lipschitz constant of theh surrogate model on the effectiveness of its adversarial examples on the target ensembles in a black-box attack. We clip the spectral norm of each layer of each of the source models to a specific value and evaluate the average robust accuracy of each of the three models. We try different layer-wise clipping values ($0.8$, $1.0$, $1.2$, and $1.5$) and for each setting compute the average of the robust accuracy over multiple random seeds. As the table shows the adversarial examples generated on the clipped models are more effective in black-box attacks, but still \texttt{LOTOS} has the highes robust accuracy compared to others.

\begin{table*}[ht!]
\begin{center}
\begin{small}
\begin{sc}
\begin{tabular}{@{} l  c c  c    @{}}
 \toprule

   &  \texttt{Orig} & $C=1.0$ & \texttt{LOTOS}
 \\\addlinespace[0.3em]

    \cmidrule(r){2-4}

 \texttt{Orig} & $19.3 \pm 1.34$ & $39.0 \pm 0.85$ & $\bf 43.8 \pm 1.30$   
 \\\addlinespace[0.3em]

 $C=0.8$ & $15.7 \pm 1.03$  & $12.9 \pm 0.64$ & $\bf 18.5 \pm 0.79$ 
 \\\addlinespace[0.3em]

 $C=1.0$ & $13.6 \pm 0.89$  & $13.9 \pm 0.30$ &  $\bf 20.5 \pm 0.69$
 \\\addlinespace[0.3em]

 $C=1.2$ & $13.2 \pm 0.96$  &  $15.9 \pm 0.69$  &  $\bf 22.9 \pm 0.91$
 \\\addlinespace[0.3em]

 $C=1.5$ & $13.1 \pm 0.65$ & $19.8 \pm 0.29 $ & $\bf 27.1 \pm 0.66$ 
 \\\addlinespace[0.4em] 

\bottomrule
\end{tabular}
\caption{\footnotesize {\bf Robust accuracy against black-box attacks.} The source models are either original models or clipped models with different layer-wise clipping value for the spectral norm. The target models are ensembles of three ResNet-18 models on CIFAR-10 for three different cases; original models (\texttt{Orig}), clipped ones ($C=1.0$), and trained with \texttt{LOTOS}. As the table shows attacks using clipped models are stronger but still \texttt{LOTOS} achieves the highest robust accuracy. }
\label{tab:nobn-cifar}
\end{sc}
\end{small}
\end{center}
\vskip -0.1in
\end{table*}

\end{document}